\documentclass[11pt, reqno]{amsart} 
\usepackage{amsxtra}
\usepackage{smartdiagram}
\usepackage{tikz}
\usetikzlibrary{calc,matrix,arrows,shapes}
\usepackage{pgfplots}
\tikzstyle{bubble}=[align=center, draw,inner sep=0pt,shape=ellipse,minimum height=1.25cm, minimum width=3cm]
\theoremstyle{plain}

\usepackage{graphicx} 

\usepackage{booktabs} 
\usepackage{array} 
\usepackage{paralist} 
\usepackage{verbatim} 
\usepackage{subfig} 
\usepackage{hyperref}
\usepackage{tikz}
\usepackage{lipsum}
\usepackage{dsfont}

\def\sideremark#1{\ifvmode\leavevmode\fi\vadjust{\vbox to0pt{\vss
\hbox to 0pt{\hskip\hsize\hskip1em%
\vbox{\hsize2cm\tiny\raggedright\pretolerance10000%
\noindent {\color{red}{#1}}\hfill}\hss}\vbox to8pt{\vfil}\vss}}}%

\usepackage{amsmath}
\usepackage{amssymb}
\usepackage{amsthm}
\usepackage{mathrsfs}
\usepackage{enumerate}
\usepackage[all]{xy}
\usepackage{todonotes}
\usepackage{xkvltxp}
\usepackage{fixme}
\usetikzlibrary{calc}

\newtheorem{thm}{Theorem}[section]
\newtheorem{prop}[thm]{Proposition}
\newtheorem{cor}[thm]{Corollary}

\newtheorem{lem}[thm]{Lemma}
\theoremstyle{definition}

\newtheorem{defn}[thm]{Definition}
\newtheorem{ex}[thm]{Example}

\newtheorem{remk}[thm]{Remark}

\newcommand{\ol}[1]{\overline{#1}}
\newcommand{\wt}[1]{\widetilde{#1}}
\newcommand{\wh}[1]{\widehat{#1}}

\newcommand{\bbq}{\mathbb Q}

\newcommand{\supp}{\textnormal{supp}}

\begin{document}

\title[Reconstructing Functions on Abelian Groups with Autocorrelations] 
{Reconstructing Rational Functions on Finite Abelian Groups with Higher Autocorrelations}
\author[W. Riley Casper]{W. Riley Casper}
\email{wcasper@fullerton.edu}

\author[Bobby Orozco]{Bobby Orozco}
\email{orozcobobby03@csu.fullerton.edu}
\address{
Department of Mathematics \\
California State University \\
Fullerton, CA 92831 \\
U.S.A.}

\date{}
\keywords{finite groups, group characters, discrete Fourier transform, X-ray crystallography, separating invariants}
\subjclass[2020]{05C69, 05C40, 20C30,13A50}
\begin{abstract}
The higher-order autocorrelations of integer-valued or rational-valued functions on finite Abelian groups appear naturally in X-ray crystallography, and have applications in computer vision systems, correlation tomography, correlation spectroscopy, and pattern recognition.
In this paper, we consider the problem of reconstructing a rational-valued function on finite Abelian groups from its higher-order autocorrelations.  We describe an explicit reconstruction algorithm, and prove that the autocorrelations up to order $3r+3$ are always sufficient to determine the data up to translation, where $r$ is the rank of the group.
We also provide examples of rational-valued functions on finite Abelian group which are not determined by their autocorrelations up to order $3r+2$.
In particular, we provide a sharp upper bound on the separating degree of the regular representation of a finite Abelian group in terms of its rank.  
\end{abstract}

\maketitle
\section{Introduction}
In this paper, we consider the problem of uniquely determining rational-valued functions on a finite Abelian group $G$ with rank $r$ (ie. rational data in an $r$-dimensional periodic grid) from its higher autocorrelations. 
Mathematically, we represent elements of $G$ as $r$-tuples of integers
$$x=(x[1],x[2],\dots,x[r]),$$
and data on $G$ as a function $f(x)$ of these $r$-tuples.  Autocorrelations of a data set are correlations between the set of data and a shift of the data, ie. between $f(x)$ and $f(x+y)$.
More generally, the \textbf{higher-order autocorrelations} are defined to be
$$M_{n}(f;x_1,\dots, x_{n-1}) = \sum_{y\in G} f(y)f(y+x_1)\dots f(y+x_{n-1}),$$
for $n\geq 1$ and $x_1,\dots, x_{n-1}\in G$.
Equivalently, functions on $G$ can be viewed as autocorrelations of periodic functions on $\mathbb R^r$ with finite Fourier series expansions.
Autocorrelations cannot distinguish between a function and its translations, so the precise problem we consider is what order of autocorrelation data is required to determine a function up to a translation (ie. a global phase).
This fits into a wider class of problems on finding degree bounds for separating invariants for finite Abelian groups.  In particular, it is equivalent to obtaining the separating degree $d$ of the regular representation of a finite Abelian group $G$.  Similar calculations are explored in \cite{kemper,kohls,domokos1,domokos2}.

Higher-order autocorrelations are useful in computer vision applications, such as gate recognition, object tracking, or video
surveillance \cite{otsu,hidaka}.
They also have applications in tomography \cite{zhang}, spectroscopy \cite{shang}, and pattern recognition \cite{mclaughlin}.
In particular, machine learning algorithms training on higher correlation data can be taught to recognize key features, such as the number of holes.  Vision chips using higher-order local autocorrelations (HLAC), ie. values of $M_{n}(f;x_1,\dots, x_{n-1})$ with $n\leq 3$ and with a $3\times 3$ mask pattern (ie. $|x_j|\leq 1$) have been built \cite{yamamoto}.
Autocorrelations of order up to $8$ and with larger mask patterns have also been studied \cite{toyoda}.  This raises the question of what order autocorrelations must be computed to uniquely determine an image. 

Historically, higher-order autocorrelations also appear in X-ray crystallography, where they are referred to as higher-order moments.
The central problem of X-ray crystallography is to determine a crystal lattice structure directly from observations of X-ray diffraction data.
In a basic experiment, one fires a monochromatic $X$-ray beam at a material sample from several different angles and observes the resulting diffraction patterns.
Using these observations, one obtains estimates of the intensities, ie. the squared norms $|F_{j,k,\ell}|^2$ of the coefficients of the electron density's Fourier expansion
$$\rho(x,y,z) = \sum_{j,k,\ell} F_{j,k,\ell} e^{-2\pi i(jx + ky + \ell z)}.$$
From there one attempts to reconstruct $\rho$ and deduce the atomic structure.

In certain cases, in particular in the presence of a few dominant atoms, this is possible via analysis of the Patterson map, eg. the Fourier transform of the intensity \cite{patterson1,patterson2}.
However, in general knowledge of the intensities is insufficient to determine the electron density, as first observed by Pauling and Shappell when studying arrangements of metal atoms in bixbyite \cite{pauling}.  
To further distinguishing atomic structures, Harker and Kasper proposed to use inequalities and direct observations of diffraction data to obtain estimates of the third-order moments, ie. third-order autocorrelation functions \cite{harker}.  This work was later refined by Karle and Hauptman and led in part to a Nobel Prize.

However, in the end even third-order autocorrelations may not be enough.
Different functions on a finite Abelian group can have the same exact autocorrelations up to order three or higher, without being translations of one-another.
As one key example, whenever $(a,b,c)$ is an integer point on the conic surface
$$x^2+3y^2=z^2$$
the autocorrelations of the function on the cyclic group $\mathbb Z/6\mathbb Z$
\begin{center}
\begin{tabular}{c|c|c|c|c|c|c}
$x$ & $0$ & $1$ & $2$ & $3$ & $4$ & $5$\\\hline
$f(x)$ & $2a$ & $a-3b$ & $-a-3b$ & $-2a$ & $3b-a$ & $a+3b$
\end{tabular}
\end{center}
up to order $5$ depend \emph{only} on the value $c$ of the triple $(a,b,c)$ (see Example \ref{ex:Z6example} for details).  Therefore for fixed $c$ we can obtain functions with the same higher autocorrelations.

One-dimensional binary data in particular has been a topic of intense research. 
Two binary data sets with the same second-order autocorrelations are called homometric cyclotomic sets and initially studied by Patterson \cite{patterson3}.
Homometric cyclotomic sets arise naturally from cyclic difference sets in additive combinatorics \cite{baumert}, as well as from $Z$-relations in music theory \cite{schuijer}, and from various arithmetic functions in number theory.

While this paper focuses on reconstructing functions on finite abelian groups (ie. discrete gridded data), one should note there is a continuous version of this problem, which is discussed in \cite{stieltjes1,stieltjes2, hamburger, shohat}, with modern treatments appearing in \cite{mead, provost, schmudgen}. 

In \cite{grunbaum}, Gr\"{u}nbaum and Moore proved that knowledge of the autocorrelations up to order six determines a rational-valued function $f$ on any cyclic group $G$ up to translation.
In fact, when $|G|$ (the number of elements of $G$) is odd the autocorrelations up to order $4$ are sufficient. 
This on its face is surprising, since the order of the autocorrelations required is independent of the size of the group. 
It also starkly contrasts with the non-finite or continuous settings.
Integrable functions on $\mathbb R$ with compact support are determined by their third-order autocorrelations \cite{bartelt,yellott}. 
However, in general functions on $\mathbb{R}$ can have the same autocorrelations to arbitrarily high order, without being translations of each other \cite{chazan}.  

Of course, in many different applications such as image recognition, data is encoded in a higher-dimensional rectangular grid.
In this setting even higher-order autocorrelations are required, and the need for an extension of Gr\"{u}nbaum and Moore's result to non-cyclic groups is apparent.
In particular, if we know what order is required to identify functions, it could better inform us on what autocorrelations to use in feature selection.
Even better, one would desire an algorithm for reconstructing a function from its autocorrelation data.

In this paper, we obtain a sharp bound for what order of autocorrelations are necessary to uniquely determine data on a function on arbitrary finite Abelian groups (ie. higher-dimensional grids), extending Gr\"{u}nbaum and Moore's result to the non-cyclic setting.
We also describe an explicit algorithm for reconstructing the data from its autocorrelations, up to a global phase.
Such an algorithm was not previously known, even in the cyclic (one-dimensional) case.

Up to isomorphism, any finite Abelian group is of the form

\begin{equation}\label{eqn:structuretheorem}
    G = (\mathbb{Z}/a_1\mathbb{Z})\times (\mathbb{Z}/a_2\mathbb{Z})\times\dots\times (\mathbb{Z}/a_r\mathbb{Z})
\end{equation}

for a unique sequence of positive integers $a_1,\dots, a_r$ with the property that $a_j | a_k$ whenever $1\leq j < k\leq r$.  The integers $a_1,\dots,a_r$ are called the \textit{invariant factors} of $G$, and the largest one $a_r$ is called the \textit{exponent} of $G$ and denote by $\exp(G)$.

By leveraging the action of the Galois group of field automorphisms of an extension of $\bbq$ by an $\exp(G)$'th root of unity, we can prove our main result saying that the rank of the group determines the number of required autocorrelations.
\medskip
\\\noindent
{\bf{Main Theorem.}} {\em{
Let $G$ be a finite Abelian group with rank $r$ (ie. an $r$-dimensional rectangular grid with periodic boundary conditions), and let $f$ be a rational-valued function on $G$.
If $|G|$ is odd, then $f$ can be reconstructed from its autocorrelations, up to order $2r+2$. If $|G|$ is even, then it can be reconstructed with autocorrelations up to order $3r+3$.
}}

We show that our bound is sharp in the sense that for any integer $r>0$ there exists rational-valued functions $f$ and $g$ on a finite Abelian group of rank $r$ of a certain size with the same autocorrelations up to order $3r+2$, but which are not translations of each other (see Example \ref{ex:kahuna}).
In particular, the theorem suggests that in some situations HLAC of an image up to order $9$ can still contain interesting data not seen by HLAC of order $8$ (with the caveat of potential discrepancies about how the boundaries of images are handled).

The proof of our Main Theorem is constructive: we use the data of the autocorrelations up to order $2r+2$ or $3r+3$ to explicitly reconstruct the function, up to a translation. Therefore we not only generalize Gr\"{u}nbaum and Moore's result in the cyclic setting, but we also fill an important gap in the literature on how to reconstruct a rational-valued function from its autocorrelation data.

The primary application of our Main Theorem and the associated constructive algorithm is to non-generic signal reconstruction.  There are many well-known results showing that for ``generic" signals, ie. those excluding a set of exceptional poorly behaved signals, reconstruction can be done using autocorrelation up to order $3$ \cite{bandeira,edidin}.  In contrast, our theorem can be applied to ``worst-case" signals, where the discrete data to be reconstructed is highly structured.  Examples include binary images, integer-valued occupancy patterns, periodic codewords, and finite libraries of gridded templates. Such signals often possess symmetries, repeated motifs, or sparse Fourier support, like the images of the two-dimensional examples in Figure \ref{fig:Z30Z_squared_f_and_g}.  This figure, depicting a special case of Example \ref{ex:kahuna}, describes two functions which are not translations of one-another, but whose autocorrelations agree up to order $8$.
Sparse Fourier support in particular is a problem in multi-reference realignment, where bispectrum reconstruction algorithms commonly assume the Fourier coefficients are bounded away from zero \cite{perry}.

The need to consider worst-case signals arises naturally in certain situations.  Problems like the combinatorial $k$-deck problem deal with \textit{arbitrary} subsets of a finite group, so analysis of all signals, rather than just generic ones, is necessary.
Our results could also have future applications to machine learning, particularly to the study low-dimensional invariant embeddings that retain enough information to support universal invariant learning, as for example in \cite{dym}. In a rational-valued version of this setting, our autocorrelation bounds provide an explicit information-complete invariant representation for rational-valued periodic arrays.

\subsection{Connection with separating invariants}

Let $G$ be a group and $V$ be a finite dimensional vector space over a field $\mathbb F$ with a linear action by the group $G$, ie. a finite dimensional $G$-representation over $\mathbb F$.  The \textbf{orbit} of $v\in V$ is the set
$$G\cdot v = \{g\cdot v: g\in G\}.$$ 
The algebra of $\mathbb F$-valued polynomial functions on $V$ is denoted by $\mathbb F[V]$.  A polynomial $p\in\mathbb F[V]$ is called \textbf{$G$-invariant} if $p(g\cdot v) = p(v)$ for all $v\in V$ and $g\in G$.  The subalgebra of $G$-invariant polynomials is denoted by $\mathbb F[V]^G$.
We say a subset $S\subseteq \mathbb F[V]^G$ is \textbf{separating} if for all $v,w\in V$
$$G\cdot v\neq G\cdot w\Longleftrightarrow \exists p\in S: p(v)\neq p(w).$$
It is well-known that $\mathbb F[V]^G$ itself is separating.  Even better, the Artin-Tate Lemma implies that if $G$ is a finite group then $\mathbb F[V]^G$ is finitely generated.  Consequently, there will exist a degree $d$ such that the set of polynomials in $\mathbb F[V]^G$ of degree at most $d$ is separating.  We call the smallest such value of $d$ the \textbf{separating degree} of $V$. The maximum separating degree over all finite dimensional representations of $G$ is called the \textbf{separating Noether number} of $G$ and denoted by $\beta_{\text{sep}}^{\mathbb F}(G)$ and is studied in \cite{kemper,kohls,domokos1,domokos2}.

In this paper, we are considering the special case of the regular representation of a finite Abelian group $G$ over $\mathbb Q$.  In other words, we are considering the vector space $V = \{f: G\rightarrow\mathbb Q\}$ of functions from $G$ to $\mathbb Q$ and the action by translation
$$y\cdot f(x) = f(x + y),\quad y\in G.$$
The set of autocorrelations up to order $d$ forms a spanning set of the space of $G$-invariant polynomials of degree at most $d$, so the Main Theorem is determining an upper bound for the separating degree of the regular representation of $G$.  As Domokos proves in \cite{domokos2}, the separating degree of the regular representation of a finite group is equal to the separating Noether number, so this is also an upper bound for the separating Noether number, ie. for any finite Abelian group $G$
\begin{equation}\label{eqn:betabound}
\beta_{\text{sep}}^{\mathbb Q}(G)\leq \left\lbrace\begin{array}{ll}
3r + 3,& G\ \text{even of rank $r$}\\
2r + 2,& G\ \text{odd of rank $r$}
\end{array}\right..
\end{equation}
Moreover, for each rank $r$ we provide examples of Abelian groups for which this is an equality.
This pairs well with Domokos's result \cite{domokos2}, where he proves $\beta_{sep}^{\mathbb Q}(\mathbb Z/p\mathbb Z) = 3$ for any odd prime $p$.

It's worth mentioning our result should be viewed as an analysis of the orbit separation problem in the worst-case setting, where we consider all possible potential signals.  Recent analyses also focus on determining generic signals, where one might exclude a collection of disallowed signals with bad behavior.  Many results in the generic setting prove that having moments up to the third order, equivalently polynomials up to degree $3$, are sufficient to recover the signal up to its orbit, in particular for the left regular representation of a finite group \cite{bandeira,edidin}.

\subsection{Mathematical tools}
There are two main mathematical ingredients that go into our reconstructive process: discrete Fourier transforms and field automorphisms.

The discrete Fourier transform $\wh f$ of a function $f(x)$ is a complex-valued function on $G$ defined by
$$\wh f(x) = \sum_{y\in G} f(y)\exp\left(-2\pi i \sum_{k=1}^r\frac{x[k]\cdot y[k]}{a_k}\right).$$
Here the $a_k$ values are the invariant factors of $G$ coming from Equation \ref{eqn:structuretheorem}, and can be viewed as the number of points in our $r$-dimensional grid in the $k$'th direction. 
More generally, the $n$-dimensional discrete Fourier transform is defined by computing the discrete Fourier transform one variable at a time.

The discrete Fourier transform of the $n$'th moment function $M_n(f;x_1,\dots, x_{n-1})$
is the product
$$\wh M_n(f;x_1,\dots, x_{n-1}) = \wh f(x_1)\wh f(x_2)\dots\wh f(x_{n}),$$
where here $x_n=-(x_1+x_2+\dots+x_{n-1})$.
Furthermore, constructing $f(x)$ up to a translation by some value $y$ is equivalent to constructing $\wh f(x)$ up to multiplication by the exponential function
$$\chi(x,y) = \exp\left(2\pi i \sum_{k=1}^r\frac{x[k]\cdot y[k]}{a_k}\right).$$
Such a function is called a \textbf{character} or \textbf{global phase factor}.   

Field automorphism are our second important mathematical tool.
Since $f$ is rational-valued, the values of $\wh f$ belong to the cyclotomic field extension of $\mathbb Q$ by $\xi_N$, where $N=\exp(G)$ is the \textbf{exponent} of $G$, which is defined as
$$\exp(G) = \min\{n > 0: nx=0\ \text{for all}\ x\in G\}.$$
This is the field consisting of all linear combinations of $N$'th roots of unity with rational coefficients
$$\mathbb{Q}(\xi_N) = \{c_0+c_1\xi_N^1+\dots+c_{N-1}\xi_N^{N-1}:c_0,\dots,c_{N-1}\in\mathbb Q\},$$
where $\xi_N=e^{2\pi i/N}$.
For any $a\in\mathbb Z/N\mathbb{Z}$ which is relatively prime to $N$, there is a unique field automorphism $\sigma_a$ of $\mathbb{Q}(\xi_N)$, defined by
\begin{equation}\label{eqn:automorphism}
\sigma_a: \sum_{k=0}^{N-1}c_k\xi_N^k \mapsto  \sum_{k=0}^{N-1}c_k\xi_N^{ak}.
\end{equation}
For example, $\sigma_1$ is the identity function $\sigma_{N-1}$ is complex conjugation; from Galois Theory, the condition that $f$ is rational-valued is then equivalent to the condition that
$$\wh f(ax)=\sigma_a(\wh f(x))$$
for all $0\leq a < N$ relatively prime to $N$.

\subsection{Basic algorithm}
We first describe our algorithm with examples in the simplest setting, when $\wh f$ is nonzero on a nice generating set of the finite Abelian group $G$.  Following our examples, we will explain what can go wrong and why we might sometimes require autocorrelations as high as order $3r+3$.
In the nice case, the recovery of $\wh f(x)$ from its autocorrelations is not too complicated, and can be done with autocorrelations of order no  higher than $3r+1$.  In fact, if $\exp(G)$ is prime, it can be done with the autocorrelations of order at most $\max(3,r+1)$.  The algorithm proceeds in the following basic steps.
\medskip\\
\noindent
\textbf{Basic Algorithm:}
\begin{enumerate}[\textbf{Step}\ 1.]
\item Calculate an integer power of $\wh f$ at the grid points $e_1=(1,0,\dots,0),\dots,e_r=(0,0,\dots, 1)$ using the formula
{\small
$$
\wh f(e_j)^{2^{\phi(N)}-1}
   = \prod_{k=0}^{\phi(N)-1}\left(\frac{\wh M_{3}(f;2^ke_j,2^ke_j)}{\wh M_2(f;2^{k+1}e_j)}\right)^{2^{(\phi(N)-1-k)}}
$$
}
\hspace{-0.6em}
where here $\phi(N)$ is Euler's phi function.  If $N$ is even, we instead use unital decompostions of powers of $2$, following the formula in Theorem \ref{thm:recon-evenpow} below.  (See also Definition \ref{def:unital}.)
\item For each $e_k$, find a rational combination
$$\alpha_k = \sum_{j=0}^{N-1} c_{kj}\xi_N^{j}$$
satisfying
$$\alpha_k^{2^{\phi(N)}-1} = \wh f(e_j)^{2^{\phi(N)}-1}.$$
\item Define a function $\wh g(x)$ using the formula
$$\ol{\wh g(x)} = \frac{\wh M_{r+1}(f;x[1]e_1,\dots,x[r]e_{r})}{\sigma_{x[1]}( \alpha_1)\sigma_{x[2]}( \alpha_2)\dots \sigma_{x[r]}( \alpha_r)}$$
on the values of $x\in G$ with $x[k]$ relatively prime to $N$ for all $1\leq k\leq r$.
If some $x[j]$ is nonzero and not relatively prime to $N$, we replace $x[j]e_j$ with three terms $ue_j$, $ve_j$ an $we_j$ in the numerator (increasing the order of the moment), where $x[j]=u+v+w$ is a unital decomposition (see Definition \ref{def:unital}).  Likewise, we replace $\sigma_{x[j]}(\alpha_j)$ in the denominator with $\sigma_u(\alpha_j)\sigma_v(\alpha_j)\sigma_w(\alpha_j)$.
If some entries in the moment are zero, we drop them (using a lower-order moment) and drop the corresponding $\alpha$ in the denominator.
Lastly, we define $\wh g(0) = M_1(f)$.
The inverse Fourier transform then returns a function $g$ which is a translation of $f$.
\end{enumerate}
With this algorithm in mind, we demonstrate two examples.
\begin{ex}
Imagine we are trying to find an unknown rational-valued function $f$ on a cyclic group with seven points $G = \{0,1,\dots, 6\}$.
Suppose we are given that the first-order autocorrelation is $M_1(f) = 3$, the second-order autocorrelation data is
$$M_2(f,x) = \left\lbrace\begin{array}{cc}
1, & x\neq 0,\\
3, & x = 0
\end{array}\right.,$$
and that the third-order autocorrelation data is
$$M_3(f,x) = \left\lbrace\begin{array}{cc}
3, & x = y = 0\\
1, & x=0\ \text{xor}\ y=0\\
1, & \{x,y\}=\{1,5\},\{2,3\},\ \text{or}\ \{4,6\}\\
0, & \text{otherwise}
\end{array}\right..$$
Then the transformed second-order autocorrelation data is
$$\wh M_2(f,k) = \left\lbrace\begin{array}{cc}
2, & k\neq 0,\\
9, & k = 0
\end{array}\right.,$$
and the third-order autocorrelation data is
$$\wh M_3(f,j,k) = \left\lbrace\begin{array}{cl}
27, & j=k=0,\\
6, & j\ \text{xor}\ k=0,\\
\xi_7^{j+3k} + \xi_7^{2j+6k} + \xi_7^{4j+5k}\\ + \xi_7^{3j+k} + \xi_7^{6j+2k} + \xi_7^{5j+4k}, & \text{otherwise}.
\end{array}\right.$$
Note in particular, this implies that $\wh f(0)^2=9$ and $\wh f(0)^3=27$, so $\wh f(0) =3$.

Since $2^3 = 1\mod 7$, we can actually get away with replacing $\phi(7)$ with $3$ in Step 1 (Theorem \ref{thm:recon-oddpow} and the remark following the theorem).  
Then we calculate
\begin{align*}
\wh f(1)^{7}
& = \prod_{k=0}^{2}\left(\frac{\wh M_{3}(f;2^k,2^k)}{\wh M_2(f;2^{k+1})}\right)^{2^{2-k}}\\
& = \prod_{k=0}^{2}\left(\xi_7^{2^{k}} + \xi_7^{2^{k+1}} + \xi_7^{2^{k+2}}\right)^{2^{2-k}}\\
& = 
\left(\xi_7^1 + \xi_7^{2} + \xi_7^{4}\right)^4
\left(\xi_7^{2} + \xi_7^{4} + \xi_7^{8}\right)^2
\left(\xi_7^{4} + \xi_7^{8} + \xi_7^{16}\right)\\
& = 
\left(\xi_7^1 + \xi_7^{2} + \xi_7^{4}\right)^7.
\end{align*}
This makes Step 2 of finding $\alpha_1$ easy.
We take
$$\alpha_1 = (\xi_7^1 + \xi_7^{2} + \xi_7^{4}).$$
Then using the formula in Step 3:
\begin{align*}
\ol{\wh g(a)} = \frac{\wh M_2(f;a)}{\sigma_a(\alpha_1)} = \frac{2}{\xi_7^a + \xi_7^{2a} + \xi_7^{4a}} = \xi_7^{-a} + \xi_7^{-2a} + \xi_7^{-4a}.
\end{align*}
Therefore
$$\wh g(a) = \xi_7^a + \xi_7^{2a} + \xi_7^{4a}.$$
Also, the first autocorrelation says that $\wh g(0) = M_1(f) = 3$.
Taking the inverse Fourier transform, we find the data we were trying to recover must have been
\begin{center}
\begin{tabular}{c|c|c|c|c|c|c|c}
$x$ & 0 & 1 & 2 & 3 & 4 & 5 & 6\\\hline
$g(x)$ & 0 & 0 & 0 & 1 & 0 & 1 & 1
\end{tabular}
\end{center}
up to a cyclic translation.
It is easy to check that the function we found satisfies the specified moment data.
\end{ex}

\begin{ex}
As a more complicated demonstration of our algorithm in practice, consider the problem of recovering an image like the one in Figure \ref{fig:crab}.  Our problem size is kept small for sake of demonstration.
\begin{figure}[htp]
    \centering
    \includegraphics[width=0.42\linewidth]{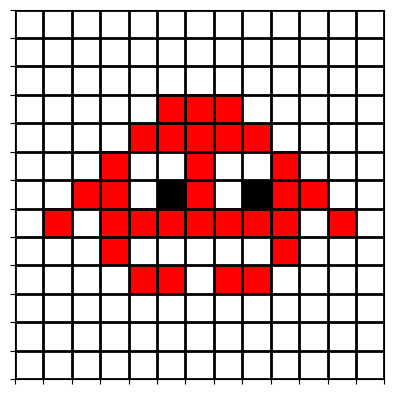}
    \caption{A basic pixel image of a crab.  This can be encoded as data on a $13\times 13$ grid with $1$'s representing red values and $2$'s representing black values and zeros everywhere else.}
    \label{fig:crab}
\end{figure}
If we represent each pixel with a numerical value, such as by taking the grayscale value, then the image is naturally encoded as data on a grid.  For simplicity sake, we encode the image in Figure \ref{fig:crab} as a matrix of $0$'s, $1$'s and $2$'s.
{\small
$$\left[\begin{array}{ccccccccccccc}
0&0&0&0&0&0&0&0&0&0&0&0&0\\
0&0&0&0&0&0&0&0&0&0&0&0&0\\
0&0&0&0&0&0&0&0&0&0&0&0&0\\
0&0&0&0&0&1&1&1&0&0&0&0&0\\
0&0&0&0&1&1&1&1&1&0&0&0&0\\
0&0&0&1&0&0&1&0&0&1&0&0&0\\
0&0&1&1&0&2&1&0&2&1&1&0&0\\
0&1&0&1&1&1&1&1&1&1&0&1&0\\
0&0&0&1&0&0&0&0&0&1&0&0&0\\
0&0&0&0&1&1&0&1&1&0&0&0&0\\
0&0&0&0&0&0&0&0&0&0&0&0&0\\
0&0&0&0&0&0&0&0&0&0&0&0&0\\
0&0&0&0&0&0&0&0&0&0&0&0&0
\end{array}\right].$$
}
Now imagine that we only had the \textit{autocorrelations} of the data of order up to three.
Using the formula in Step 1 above, we calculated the value of $\wh f(e_j)^{2^{12}-1}$ to very high precision for $e_1 =(1,0)$ and $e_2 = (0,1)$.
$$\wh f(e_1)^{4095} = 1.78868\cdot 10^9 - 1.39674\cdot 10^{10}i$$
$$\wh f(e_2)^{4095} = 1.49619\cdot 10^9 + 8.68233\cdot 10^8i.$$
The actual values stored in the computer were calculated using very high precision with Mpmath in Python \cite{mpmath}.

For Step 2, we computed candidates for $\alpha_1$ and $\alpha_2$ by taking an appropriate root and using the PSLQ algorithm \cite{ferguson} implemented in Mpmath to identify integer coefficients expressing $\alpha_1$ and $\alpha_2$ as combinations of powers of $\xi_{13}$, as in step $2$.
Note that computing a $(2^{12}-1)$'th root initially gives a value outside the field $\mathbb Q(\xi_{13})$, but this can be fixed by multiplying by an appropriate $(2^{12}-1)$'th root of unity.
The issues of identifying the linear combination or ending up in the wrong field could have instead been resolved by doing exact-integer arithmetic with an explicit computer implementation of the cyclotomic field extension and being careful when obtaining the root.

For our final step, we define a function
$$\ol{\wh g((a,b))} = \frac{\wh M_{3}(f; ae_1, be_2)}{\sigma_a(\alpha_1)\sigma_b(\alpha_2)},$$
for each nonzero value of $a$ and $b$.
Furthermore, we define
$$\ol{\wh g((a,0))} = \frac{\wh M_{2}(f; ae_1)}{\sigma_a(\alpha_1)},\ \ \ \ol{\wh g((0,b))} = \frac{\wh M_{2}(f; be_2)}{\sigma_b(\alpha_2)},$$
and
$$\wh g((0,0)) = M_1(f).$$
Taking the inverse Fourier transform, we obtained a cyclically shifted version of the same matrix we started with. 
\end{ex}

\subsection{When things go wrong}
The problem of recovering $\wh f$ from its moments becomes much more complicated when $\wh f$ vanishes at various points.
In particular, many of the expressions we used above would no longer make sense due to divide-by-zero issues.
In this situation, autocorrelations of order up to $3r+3$ may be required to uniquely determine $\wh f$ up to a global phase.

The value of $\wh f(x)^n$ for some positive integer $n$ can be easily determined by the autocorrelations up to order $6$ (see Theorem \ref{thm:recon-oddpow} and Theorem \ref{thm:recon-evenpow}).
Therefore the remaining issue of finding $\wh f$ is one of determining $\wh f$ from $\wh f^n$, amounts to choosing the right $n$'th root of the expression we can calculate.  Put simply, we can get the value of $\wh f$ at each point up to a multiple by an $n$'th root of unity, which we call the \textbf{local phase}.
In this way, the heart of the difficulty in reconstructing $f$, or equivalently, $\wh f$ is in determining the local phases.

In our basic algorithm, determining a consistent local phases was easy since we assumed we had nonzero values on a simple set of generators for our group, ie. $e_1,\dots, e_r$.
When we don't, then we need to work with not-so-simple generators and things get delicate.  In this case, we must consider the prime factorization of the exponent $N=\exp(G)$ of $G$
$$N=p_1^{m_1}p_2^{m_2}\dots p_s^{m_s}.$$
For each $k$, we let $b_k = N/p_k^{m_k}$ and consider the quotient group $G/H_k$ for $H_{k} = \{x\in G: b_kx = 0\}$.
Since this is a $p_k$-group, we have control over the number of generators $r_k$ in minimal generating sets (see Lemma \ref{lem:prime generating}).  In particular, for each $k$ we can choose a set $x_{k1},\dots, x_{kr_k}$ of elements in the support of $\wh f$, which generate $G$ modulo $H_k$, where $r_k\leq r$.

Finally, we can carefully construct the local phase function over each quotient. We choose $\alpha_{k1},\dots,\alpha_{kr_k}$ analogous to Step 2 above satisfying
$$\alpha_{kj}^n = \wh f(x_{kj})^n$$
for some positive integer $n$.
This must be done delicately, so that the local phases of the $\alpha_{kj}$'s are \emph{aligned}.
By this we mean that any relations between the generators $x_{k1},\dots, x_{kr_k}$ modulo $H_k$ need to translate to analogous relations between the $\alpha_{kj}$'s.  For details, see the discussion in Section \ref{sec:reconstruction}, and in particular what we call the \textbf{remodulation condition}. 
Then we can use a formula akin to the one in Step 3 to reconstruct $\wh f(x)^{b_k}$ up to a global phase. Since the greatest common divisor of $b_1,\dots,b_s$ is $1$, they may be combined together to obtain $\wh f(x)$ up to a global phase.
This more complicated version of the algorithm is featured in Theorem \ref{thm:finalconstruction} below. 

\subsection{Group notation}
We will now introduce notation that will be used throughout the paper.  In particular, we will write $\mathbb Z/a\mathbb Z$ to mean the set
$$\mathbb Z/a\mathbb Z = \{0,1,2,\dots,a-1\},$$
equipped with the operation of addition modulo $a$.
This is exactly a one-dimensional grid with $a$ points and addition performed periodically over the boundary.
More generally, we will use
$$G = (\mathbb Z/a_1\mathbb Z)\times (\mathbb Z/a_2\mathbb Z)\times\dots\times (\mathbb Z/a_r\mathbb Z)$$
to represent an $r$-dimensional grid of points with $a_1$ many points in the first direction, $a_2$ in the second direction, and so on.
An element $x\in G$ is represented by an $r$-tuple $(x[1],x[2],\dots, x[r])$, and addition is done coordinate-wise and modulo the number of points in the given direction:
$$x+y = (x[1]+y[1]\mod a_1,\dots, x[r]+y[r]\mod a_r).$$
As above, we will define $e_k=(0,\dots, 1,\dots, 0)$ to be the vector with a $1$ in the $k$'th entry and zeros elsewhere.  

By a fundamental result in group theory called the Structure Theorem for Finite Abelian Groups, we can also assume $a_1$ divides $a_2$, that $a_2$ divides $a_3$, and so on.
In this situation, the $a_j$'s are called the \textbf{invariant factors} of $G$ and the last one $a_r$ is equal to the exponent $\exp(G)$ of $G$.  The number of invariant factors measures the smallest dimension our grid can be embedded into in a way that is compatible with the addition operation.
\section{Discrete Fourier transforms}
Our primary tool for studying the determinacy of autocorrelations on Abelian groups for functions is the discrete Fourier transform.

\subsection{Discrete Fourier transforms and autocorrelations}
\begin{defn}
Let $G$ be a finite Abelian group.
The (non-normalized) \emph{discrete Fourier transform} of a function $f: G\rightarrow\mathbb{C}$ is the function $\wh f: G\rightarrow\mathbb{C}$ defined by
$$\wh f(x) = \sum_{y\in G} f(y)\exp\left(-2\pi i \sum_{k=1}^r\frac{x[k]\cdot y[k]}{a_k}\right).$$
\end{defn}
More generally, one can define Fourier transforms on locally compact Abelian groups, compact non-Abelian groups, and in general so-called amenable topological groups.
Detailed references for the dual group can be found in \cite{hewitt} or \cite{fuchs}.

One can recover a function from its discrete Fourier transform via the inverse discrete Fourier transform
$$
f(x) = \frac{1}{|G|}\sum_{y\in G}\wh f(y)\exp\left(2\pi i \sum_{k=1}^r\frac{x[k]\cdot y[k]}{a_k}\right).$$
In particular, the discrete Fourier transform defines an isomorphism between the vector spaces of complex-valued functions on $G$ to itself.
For convenience, we will write $\chi(x,y)$ to represent the exponential above, ie. 
$$\chi(x,y) = \exp\left(2\pi i \sum_{k=1}^r\frac{x[k]\cdot y[k]}{a_k}\right).$$
Very often, we will rely on the orthogonality of sums of roots of unity
\begin{equation}\label{eqn:orthogonality}
\sum_{x\in G}\chi(x,y)\overline{\chi(x,z)} = |G|\delta_{y,z}\quad\text{for all}\ y,z\in G,
\end{equation}
where here $\delta_{y,z}$ is the Kronecker delta function.

The main utility of the discrete Fourier transform in our setting comes from the fact that it relates autocorrelations of functions to products of the function's Fourier transform, as shown in the following Proposition.
\begin{prop}\label{prop:moment2fourier}
Let $f$ be a function on a finite Abelian group $G$.  The $n$'th autocorrelation
$$M_n(f;x_1,\dots, x_{n-1}) = \sum_{y\in G}f(y)f(y+x_1)\dots f(y+x_{n-1}).$$
of $f$ can be expressed in terms of the discrete Fourier transform as
\begin{align*}
\frac{1}{|G|^{n-1}}\sum_{y_1+\dots+y_n=0}\wh f(y_1)\dots \wh f(y_n)\chi(x,y_1)\dots\chi(x,y_n),
\end{align*}
where the sum on the right is taken over all $n$-tuples $(y_1,\dots,y_n)$ which sum to $0$.
\end{prop}
\begin{proof}
For every complex-valued function $f$ on $G$, we may write
$$f(x) = \frac{1}{|G|}\sum_{y\in G}\wh f(y)\chi(x,y).$$
Therefore for $x_n=0$
\begin{align*}
M_n(f;x_1,\dots,x_{n-1}) & = \sum_{y\in G}f(y)f(y+x_1)\dots f(y+x_{n-1})\\
  & = \frac{1}{|G|^n}\sum_{y,y_1,\dots,y_n\in G}\prod_{j=1}^n\wh f(y_j)\chi(x_j+y,y_j)\\  
  & = \frac{1}{|G|^{n-1}}\sum_{y_1+\dots+y_n = 0}\prod_{j=1}^n\wh f(y_j)\chi(x_j,y_j).
\end{align*}
\end{proof}

The previous Proposition allows us to encode the condition of the $n$'th autocorrelations agreeing as a condition on the discrete Fourier transforms.

\begin{cor}
Let $f$ and $g$ be a functions on a finite Abelian group $G$. Then the $n$'th autocorrelations agree
$$M_n(f;x_1,\dots, x_{n-1}) = M_n(g;x_1,\dots, x_{n-1})$$
for all $x_1,\dots,x_{n-1}\in G$ if and only if 
$$\wh f(x_1)\dots \wh f(x_n) = \wh g(x_1)\dots \wh g(x_n)$$
for all $x_1,\dots,x_n\in G$ with $x_1+\dots+x_n=0$.
\end{cor}
\begin{proof}
By Proposition \ref{prop:moment2fourier}, if $\wh f(y_1)\dots \wh f(y_n) = \wh g(y_1)\dots \wh g(y_n)$ for all $y_1,\dots,y_n\in G$ with $y_1+\dots+y_n=0$, then the autocorrelations agree.
Conversely, suppose that the autocorrelations agree.  If $y_1,\dots,y_n\in G$, then by the orthogonality of characters
\begin{align*}
\sum_{x_1,\dots,x_{n-1}\in G} M_n (f;x_1,\dots,x_{n-1})\prod_{j=1}^{n-1}\ol{\chi(x_j,y_j)} = \wh f(y_1)\wh f(y_2)\dots\wh f(y_n). 
\end{align*}
Therefore if the $n$'th autocorrelations of $f$ and $g$ agree, then $\wh f(x_1)\dots \wh f(x_n) = \wh g(x_1)\dots \wh g(x_n)$.
\end{proof}

Using this result, we can prove that the third-order autocorrelations are sufficient to determine a function up to its phase when the Fourier transform is non-vanishing.

\begin{thm}\label{thm:basictriple}
Let $G$ be a finite Abelian group and suppose that $f$ and $g$ are two functions on $G$.
If either $\wh f$ or $\wh g$ is never zero and both agree up to the third autocorrelation, then there exists $y\in G$ satisfying
$$g(x) = f(x+y),\quad\text{for all $x\in G$}.$$
\end{thm}
\begin{proof}
Suppose that $f$ and $g$ have the same autocorrelations, up to order three.
This in particular means that for any $x\in G$,
$$\wh f(x)\wh f(-x) = \wh g(x)\wh g(-x).$$
Thus if $\wh f$ is never zero, then $\wh g$ is never zero and vice versa.  Hence both $\wh f$ and $\wh g$ are never zero and we can write
$$\frac{\wh g(-x)}{\wh f(-x)} = 
 \left(\frac{\wh g(x)}{\wh f(x)}\right)^{-1}.$$
Furthermore, for any $x_1,x_2\in  G$ we have
$$\wh f(-x_1)\wh f(-x_2)\wh f(x_1+x_2) = \wh g(-x_1)\wh g(-x_2)\wh g(x_1+x_2).$$
Therefore
$$\frac{\wh g(x_1+x_2)}{\wh f(x_1+x_2)} = \frac{\wh f(-x_1)\wh f(-x_2)}{\wh g(-x_1)\wh g(-x_2)} = \frac{\wh g(x_1)\wh g(x_2)}{\wh f(x_1)\wh f(x_2)}.$$
This shows that $\wh g/\wh f: G\rightarrow\mathbb C$ is a group character.
Consequently there exists $y\in G$ with
$\frac{\wh g(x)}{\wh f(x)} = \chi(x,y)$ for all $y\in G$.
It follows that 
$$\wh g(x) = \wh f(x)\chi(x,y),\quad\text{for all}\ x\in G,$$
and therefore $g(x) = f(x+y)$.
\end{proof}
Note in particular that in the previous theorem we made no assumptions about the values of $f$ and $g$.
The more interesting setting happens when $\wh f$ or $\wh g$ vanishes on some of the elements of $G$.
In this case, without knowledge about the values of $f$ and $g$ (eg. real, rational, etc), the autocorrelations of both $f$ and $g$ must agree up to a potentially very large order, as demonstrated by the next two examples.

\begin{remk}
Note that Theorem \ref{thm:basictriple} is not new, and has been discovered and rediscovered multiple times in the literature.  The oldest source where it is explicitly stated appears to be Adler and Konheim \cite{adler} in the broader context of locally compact Abelian groups, which include finite Abelian groups as a special case.  It was extended to non-Abelian groups using representation theory in Kakarala's thesis \cite{kakarala}; see also \cite{kakarala2}.
\end{remk}

We establish two fairly different lower bounds on the number of required autocorrelations for arbitrary functions.
Specifically, we need at least $a_r$ autocorrelations, as well as at least $(a_1+a_2+\dots+a_r)/d$ autocorrelations where $d>1$ divides $a_1$.

\begin{ex}
Let $f$ be the function whose discrete Fourier transform satisfies $\wh f(e_r) = 1$ and $\wh f(x)=0$ otherwise.
The autocorrelations of $f$ are zero up to order $\exp(G)-1$.
Since $f$ is not identically zero, $g\equiv 0$ has the same autocorrelations but is not a translation of $f$.  Thus we need autocorrelations up to order at least $\exp(G)$ to determine the value of $g$ up to translation.
\end{ex}

\begin{ex}
Let $d>1$ be a divisor of $a_1$, and let $f$ and $g$ be the functions whose discrete Fourier transforms satisfy
$$\wh f(e_k) = \wh g(e_k) = 1\quad\text{for all}\ 1\leq k\leq r,$$
and also
$$\wh f(a_1/d,\dots,a_r/d)=1\quad\text{and}\quad \wh g(a_1/d,\dots,a_r/d)=e^{2\pi i/d}$$
and $\wh f(\chi)=\wh g(\chi) = 0$ otherwise.
Suppose $x_1,\dots, x_m \in G$ with $x_1+\dots+x_m = 1$ for some $m\leq (a_1+a_2+\dots+a_r)/d$.
Since $\supp(\wh f) = \supp(\wh g)$, if any of the $x_j$'s lie outside their mutual support, then the autocorrelations $\wh f(x_1)\dots\wh f(x_m)$ and $\wh g(x_1)\dots\wh g(x_m)$ will agree.
Otherwise, if we let 
$$m_0 = \#\{j: x_j=(a_1/d,\dots, a_r/d)\}$$
and
$$m_j = \#\{j: x_j=e_j\},$$
then evaluating $x_1+\dots+x_m=1$ on the element $x=(0,\dots, 1,\dots, 0)\in G$ with a $1$ in the $j$'th position and zeros elsewhere implies $m_j+(a_j/d)m_0=0\mod a_j$. If $m_0$ is not a multiple of $d$, then this means $m_j$ is a positive multiple of $a_j/d$.  In particular, $m_j\geq a_j/d$ for all $j$, so that $m = m_0+\dots+m_r > (a_1+\dots+a_r)/d$ which is a contradiction.  Otherwise,
$$\wh f(x_1)\dots\wh f(x_m) = 1 = \exp(2\pi i m_0/d) = \wh g(x_1)\dots\wh g(x_m).$$
Thus the autocorrelations of $f$ and $g$ agree up to $(a_1+\dots+a_r)/d$.

If $g(x) = f(x+y)$ for some $y\in G$, then $\wh g(x) = \wh f(x)\chi(x,y)$ for all $x\in G$.
This means that $1= \wh g(e_k)=\wh f(e_k)\chi(e_k,y) = e^{2\pi iy_k}$ and therefore $y=(0,\dots, 0)\in G$.
This implies $\wh f=\wh g$, which contradicts the values of $\wh{f}(a_1/d,\dots, a_r/d) \textnormal{ and } \wh{g}(a_1/d,\dots, a_r/d)$.
Thus $g$ is not a translation of $f$.
This shows that at least $(a_1+a_2+\dots+a_r)/d$ autocorrelations are needed to determine an arbitrary function up to translation.
\end{ex}

\subsection{Discrete Fourier transforms and field automorphisms}
When functions are restricted to have their values appear only in certain fields, their discrete Fourier transforms exhibit interesting symmetries.
For example, it is well-known that if $f: G\rightarrow\mathbb{C}$ is a real-valued function then
$$\wh f(-x) = \overline{\wh f(x)},$$
where $\overline z$ represents the complex conjugate of $z$.  

If $f$ is constrained to take only rational values, the number of symmetries that one can observe in the transformed function grows substantially.
This comes from the fact that in such a setting $\wh f$ takes values in a so-called \emph{cyclotomic field extension} of the rationals $\mathbb Q(\xi_N)$.
This is the smallest field containing all the rational numbers along with the primitive $N$'th root of unity $\xi_N=e^{2\pi i/N}$, and is given by
$$\mathbb Q(\xi_N) = \left\lbrace\sum_{k=0}^{N-1}a_k\xi_N^k: a_0,\dots, a_{N-1}\in \mathbb Q \right\rbrace.$$

To understand a field like $\mathbb{Q}(\xi_N)$, the natural objects to study turn out to be the ``symmetries" of the field, ie. the so-called \emph{field automorphisms}.
These are the bijective functions $\sigma: \mathbb{Q}(\xi_N)\rightarrow \mathbb{Q}(\xi_N)$ which are simultaneously additive
$$\sigma(z+w) = \sigma(z)+\sigma(w),\quad \text{for all}\ z,w\in \mathbb{Q}(\xi_N),$$
as well as multiplicative
$$\sigma(zw) = \sigma(z)\sigma(w),\quad \text{for all}\ z,w\in \mathbb{Q}(\xi_N).$$

The field automorphisms of the cyclotomic field extension $\mathbb Q(\xi_N)$ are parameterized by the set
\begin{equation}\label{eqn:ring of units}
(\mathbb Z/N\mathbb Z)^\times = \{a\in\mathbb Z/N\mathbb Z: \gcd(a,N)=1\}
\end{equation}
of elements of $\mathbb Z/N\mathbb Z$ relatively prime to $N$.
This is called the \emph{multiplicative group of units} of the ring $\mathbb Z/N\mathbb Z$.
For each $a\in \mathbb Z/N\mathbb{Z}$ there exists a unique field automorphism $\sigma_a$ defined by Equation \eqref{eqn:automorphism}, ie. 
$$\sigma_a\left(\sum_{k=0}^{N-1}c_k\xi_N^k\right) = \sum_{k=0}^{N-1}c_k\xi_N^{ak}.$$
Moreover, every field automorphism of $\mathbb Q(\xi_N)$ appears as one of these.

Using the field automorphisms of $\mathbb Q(\xi_N)$, we can describe all the symmetries satisfied by discrete Fourier transforms of rational functions.
\begin{thm}\label{thm:galois}
Suppose $G$ is a finite Abelian group, and let  $N=\exp(G)$ and $\xi_N$ be a primitive $N$'th root of unity.  Assume $f: G\rightarrow\mathbb{C}$ is a function whose values lie in $\mathbb{Q}(\xi_N)$.
Then $f$ is rational-valued if and only if
$$\wh f(ax) = \sigma_a(\wh f(x))$$
for all $x\in G$ and $a\in (\mathbb{Z}/N\mathbb{Z})^\times$.
\end{thm}
\begin{proof}
Using the discrete Fourier transform, we write
$$f(x) = \frac{1}{|G|}\sum_{y\in G}\wh f(y)\chi(x,y).$$
Moreover, for any $x,y\in G$ the value of $\chi(x,y)$ is an $N$'th root of unity for all $x\in G$ and thus $\sigma_a(\chi(x,y)) = \chi(x,y)^a$.

If $f(x)$ is rational-valued, then for any $a\in (\mathbb{Z}/N\mathbb{Z})^\times$, we have $\sigma_a(f(x))= f(x)$.
Therefore
$$\sum_{y\in G}\wh f(y)\chi(x,y) = \sum_{y\in G}\sigma_a(\wh f(y))\chi(x,ay),$$
where we have used the fact that
$$\chi(ax,y) = \chi(x,y)^a = \chi(x,ay).$$
Orthogonality of sums of roots of unity (Equation \eqref{eqn:orthogonality}) then implies that for any $z\in G$
\begin{align*}
|G|\wh f(az) = \sum_{x,y\in G}\wh f(y)\chi(x,y)\ol{\chi(x,az)} &= \sum_{x,y\in G}\sigma_a(\wh f(y))\chi(x,ay)\ol{\chi(x,az)} \\ &= \sum_{y\in G}\sigma_a(\wh f(y))\delta_{ay,az} = |G|\sigma_a(\wh f(z)).
\end{align*}
Conversely,
suppose
$$\wh f(ay)=\sigma_a(\wh f(y))\quad \text{for all}\ y\in G, a\in(\mathbb Z/N\mathbb{Z})^\times.$$
Then for all $a\in(\mathbb Z/N\mathbb{Z})^\times$, we have
\begin{align*}
\sigma_a(f(x)) &= \frac{1}{|G|}\sum_{y\in G}\sigma_a(\wh f(y))\chi(x,ay) \\ &= \frac{1}{|G|}\sum_{y\in G}\wh f(ay)\chi(x,ay)  = \frac{1}{|G|}\sum_{z\in G}\wh f(z)\chi(x,z) = f(x).
\end{align*}
Therefore by the Fundamental Theorem of Galois theory \cite{hungerford}, we know $f(x)$ is rational-valued.
\end{proof}

\section{Reconstruction Algorithm}
The problem of reconstructing $f(x)$ is equivalent to the problem of reconstructing $\wh f(x)$.
Most often in our reconstruction, it will often be more useful to work with the discrete Fourier transform of the higher autocorrelation function, rather than the autocorrelations themselves, 
$$\wh M_n(f;x_1,\dots,x_{n-1}) = \wh f(x_1)\dots \wh f(x_{n-1}) \wh f(-x_1-\dots-x_{n-1}).$$

Moreover, if the support
$$\supp(\wh f) = \{x\in G: \wh f(x)\neq 0\}$$
does not generate the entire group itself, meaning every element of $G$ can be written as the sum a sequence of elements of $\supp(\wh f)$, then we can replace $G$ with the subgroup generated by the support.
Therefore without loss of generality, we will assume in this section that $G$ is generated by $\supp(\wh f)$.

Our reconstruction algorithm comes in two parts: an algorithm for reconstructing $\wh f(x)$ up to local phase, and a follow-up algorithm for determining consistent choices of local phases.
Here by reconstructing $\wh f$ up to local phase, we mean determining the value of $\wh f(x)$ up to multiplication by a root of unity.
This is generally much stronger than knowing the absolute value $|\wh f(x)|$, since a cyclotomic field extension contains many complex numbers with absolute value $1$ which are not roots of unity.
Rather, this is equivalent to determining $\wh f(x)$ to some integer power.
Indeed, determining $|\wh f(x)|$ can be successfully accomplished using only autocorrelations up to order $2$, while determining the value of some power of $\wh f(x)$ in general requires autocorrelations up to order $6$.

\subsection{Reconstruction up to local phase}
The first step to reconstructing a rational-valued function $f(x)$ up to a global phase is to reconstruct a power of $\wh f(x)$.
This in turn determines the value of $\wh f(x)$ itself, up to multiplication by a root of unity, ie. a local phase.

In the case that the exponent of the group $G$ is odd, the value of $\wh f(x)$ to some power is computable using only second and third-order autocorrelations.
The key strategy is to leverage Euler's Theorem, which states that
$$2^{\phi(N)} = 1\mod N,$$
whenever $N$ is odd, where here $\phi$ is Euler's phi function.
Using this, we have the following theorem.
\begin{thm}\label{thm:recon-oddpow}
Let $G$ be a finite Abelian group, and suppose that $N=\exp(G)$ is odd.  Then for any $x\in\supp(\wh f)$, we have
$$
\wh f(x)^{2^{\phi(N)}-1}
   = \prod_{k=0}^{\phi(N)-1}\left(\frac{\wh M_{3}(f;2^kx,2^kx)}{\wh M_2(f;2^{k+1}x)}\right)^{2^{(\phi(N)-1-k)}}.
$$
\end{thm}
\begin{proof}
Since $2^{k}x\in\supp(\wh f)$ for all $x\in\supp\wh f$, the denominator of the product above is nonzero, making the product well-defined.
Furthermore, the product is telescoping, since
\begin{align*}
\prod_{k=0}^{\phi(N)-1}\left(\frac{\wh M_{3}(f;2^kx,2^kx)}{\wh M_2(f;2^{k+1}x)}\right)^{2^{(\phi(N)-1-k)}}
 & = \prod_{k=0}^{\phi(N)-1}\left(\frac{\wh f(2^kx)^2}{\wh f(2^{k+1}x)}\right)^{2^{(\phi(N)-1-k)}}\\
 & = \prod_{k=0}^{\phi(N)-1}\frac{\wh f(2^kx)^{2^{(\phi(N)-k)}}}{\wh f(2^{k+1}x)^{2^{(\phi(N)-1-k)}}}\\
 & = \frac{\wh f(x)^{2^{\phi(N)}}}{\wh f(2^{\phi(N)}x)} =  \wh f(x)^{2^{\phi(N)}-1}.
\end{align*}

\end{proof}
\begin{remk}
Note that we can replace $\phi(N)$ with the order of $2$ in $(Z/N\mathbb Z)^\times$ in the above theorem, ie. the smallest positive integer $\ell$ with $2^\ell = 1\mod N$.  This can sometimes save a lot of computation.
\end{remk}
The reconstruction of a power of $\wh f$ in the case that $N$ is even is significantly more difficult.
If $N$ is not a multiple of either $3$ or $5$, then by using autocorrelations of order $4$ and $6$, one can obtain an expression similar to the one in Theorem \ref{thm:recon-oddpow}, replacing the powers of $2$ with powers of $3$ or $5$.

However, in the case that $N$ is divisible by $2$, the behavior of powers of $2$ modulo $N$ is not as nice, and also $\wh f(2x)$ may be zero, even when $\wh f(x)$ is not, significantly complicating the situation.

In the general setting, the key to our construction comes from a result in additive combinatorics on sumsets that will also play a role in the second part of our reconstruction. 
\begin{thm}[\cite{sander}]\label{thm:sumset}
Let $N$ be a positive integer.  If $N$ is odd, then any value in $\mathbb Z/N\mathbb{Z}$ can be written as a sum of two elements in $(\mathbb{Z}/N\mathbb{Z})^\times$:
$$\mathbb{Z}/N\mathbb{Z} = (\mathbb{Z}/N\mathbb{Z})^\times + (\mathbb{Z}/N\mathbb{Z})^\times.$$
If $N$ is even, then any even element in $\mathbb{Z}/N\mathbb{Z}$ can be written as a sum of two elements in $(\mathbb{Z}/N\mathbb{Z})^\times$:
$$2\mathbb{Z}/N\mathbb{Z} = (\mathbb{Z}/N\mathbb{Z})^\times + (\mathbb{Z}/N\mathbb{Z})^\times,$$
and any odd element in $\mathbb{Z}/N\mathbb{Z}$ can be written as a sum of three elements in $(\mathbb{Z}/N\mathbb{Z})^\times$:
$$(\mathbb{Z}/N\mathbb{Z})\backslash(2\mathbb{Z}/N\mathbb{Z}) = (\mathbb{Z}/N\mathbb{Z})^\times + (\mathbb{Z}/N\mathbb{Z})^\times + (\mathbb{Z}/N\mathbb{Z})^\times.$$
\end{thm}

Using this theorem, any number $c\in \mathbb Z/N\mathbb{Z}$ has what we call a unital decomposition.
\begin{defn}\label{def:unital}
A \textbf{unital decomposition} of a number $c$ in the ring $\mathbb{Z}/N\mathbb{Z}$ is a triple of numbers $u,v,w$ satisfying
$$c = u+v+w\mod N,$$
where here $u,v,w\in (\mathbb Z/N\mathbb Z)^\times$ if $N$ is even and $c$ is odd, or 
$u,v\in (\mathbb Z/N\mathbb Z)^\times$ and $w=0$ if either $N$ is odd, or both $N$ and $c$ are even.
\end{defn}

For each $k\geq 1$, let
$$2^k=u_k+v_k+w_k\mod N$$
be a unital decomposition of $2^k$ (so obviously $w_k=0$).
For simplicity, we will specifically take $u_1=v_1=1$.
Also if we let $2^m$ is the highest power of $2$ dividing $N$ and $\ell=\phi(N/2^m)$, then Euler's Theorem implies that
$$2^{k+\ell} = 2^k\mod N,$$
for all integers $k\geq m$.
Therefore we may choose
$$u_{k+\ell}=u_k\quad\text{and}\quad v_{k+\ell}=v_k.$$

With these choices in mind, we have the following theorem.
\begin{thm}\label{thm:recon-evenpow}
Let $G$ be a finite Abelian group with $N=\exp(G)$, and let $\ell$, $m$, $u_k$ and $v_k$ be defined as in the previous paragraph.
Then for any $x\in\supp(\wh f)$, we have
$$
\wh f(x)^{2^m(2^{\ell}-1)}
   = \frac{\prod_{k=1}^{\ell+m-1} (\wh B_{k}/\wh A_{k+1})^{2^{\ell+m-1-k}}}{\prod_{k=1}^{m-1} (\wh B_{k}/\wh A_{k+1})^{2^{m-1-k}}}
$$
where here
\begin{align*}
\wh A_k &= \wh M_{4}(f;u_{k}x,v_{k}x,-u_{k}x),\\
\wh B_k &= \wh M_{6}(f;u_kx,v_kx,u_kx,v_kx,-u_{k+1}x).
\end{align*}
\end{thm}
\begin{proof}
Since $u_kx,v_kx\in\supp(\wh f)$ for all $k$, the above products are well-defined.
Furthermore, we have a telescoping product
\begin{align*}
\prod_{k=1}^{\ell+m-1} \left(\frac{\wh B_{k}}{\wh A_{k+1}}\right)^{2^{\ell+m-1-k}}
 &= \prod_{k=1}^{\ell+m-1}\left(\frac{\wh f(u_kx)^2\wh f(v_kx)^2}{\wh f(u_{k+1}x)\wh f(v_{k+1}x)}\right)^{2^{\ell+m-1-k}}\\
 &= \prod_{k=1}^{\ell+m-1}\frac{\left(\wh f(u_kx)\wh f(v_kx)\right)^{2^{\ell+m-k}}}{\left(\wh f(u_{k+1}x)\wh f(v_{k+1}x)\right)^{2^{\ell+m-1-k}}}\\
 &= \frac{\left(\wh f(u_1x)\wh f(v_1x)\right)^{2^{\ell+m-1}}}{\wh f(u_{\ell+m}x)\wh f(v_{\ell+m}x)} = \frac{\wh f(x)^{2^{\ell+m}}}{\wh f(u_{m}x)\wh f(v_{m}x)}.
\end{align*}
A similar calculation shows
$$\prod_{k=1}^{m-1} \left(\frac{\wh B_{k}}{\wh A_{k+1}}\right)^{2^{m-1-k}} = \frac{\wh f(x)^{2^{m}}}{\wh f(u_{m}x)\wh f(v_{m}x)}.$$
Dividing these two expressions, we obtain the statement of our theorem.
\end{proof}

\subsection{Reconstructing the local phases}\label{sec:reconstruction}
Using the results of the previous subsection, one is able to reconstruct $\wh f(x)$, up to local phases from knowledge of a power of $\wh f(x)$.
This still leaves the problem of determining the local phases themselves.
Doing so will mean playing with much larger autocorrelations than order $4$ or $6$ potentially, so to facilitate our explanation we adopted the following convenient abuses of notation, which we will carry throughout the rest of the paper.
\begin{itemize}
\item We set $\sigma_0(\beta)=1$ for all $\beta\in \mathbb Q(\xi_N)$
\item For any zero sequence $(x_1,\dots,x_n)$ of elements of $G$ which is zero-sum, ie. $x_1+\dots+x_n=0$, the expression
$$\wh M(f; (x_1,\dots, x_n))$$
denotes the autocorrelation $\wh M_{n-\ell-1}$ of $f$, where
we have omitted all of the $x_j$'s equal to zero (so $\ell$ is the number of omitted terms)
\end{itemize}

This problem of reconstructing the local phases is quite delicate, since in general the support of $\wh f$ might not contain a ``nice" set of generators of $G$, ie. one of the right size.
For example, if $G = \mathbb Z/6\mathbb Z$, then $\{1\}$ is a generating set for $G$ and has the right size in the sense that its size is the same as the number $r$ of invariant factors.
However, $\{2,3\}$ also generates $\mathbb Z / 6 \mathbb Z$, but isn't the right size.
This problem only gets worse when the group is not cyclic, ie. $r>1$.

To fix this, we use group quotients in order to work with finite Abelian $p$-groups, ie. finite Abelian groups whose order is a power of a prime.
In such a group, every generating set contains a generating subset of the right size.
\begin{lem}\label{lem:prime generating}
Let $p$ be a prime and $G$ be an Abelian $p$-group with $r$ invariant factors.
If $S\subseteq G$ generates $G$, then $S$ contains a subset of $r$ elements which generate $G$.
\end{lem}
\begin{proof}
We proceed by induction on the rank $r$.
If $r=1$, then $G\cong \mathbb{Z}/p^k\mathbb{Z}$ for some $k>1$.
If $x\in G$ has order less than $p^k$, then $x$ belongs to the subgroup $p(\mathbb{Z}/p^k\mathbb{Z})$.
Since $S$ generates $G$, it therefore must contain at least one element $x\in S$ of order $p^k$, and thus $G$ is generated by $\{x\}$.

As an inductive assumption, suppose that any generating set of any Abelian $p$-group with $r$ invariant factors contains a generating set with $r$ elements.
Let $G$ be an Abelian group with $r+1$ invariant factors, and suppose that $S\subseteq G$ is a generating set for $G$.
Without loss of generality take
$$G = \mathbb{Z}/p^{k_1}\mathbb{Z}\oplus \mathbb{Z}/p^{k_2}\mathbb{Z}\oplus \dots\oplus \mathbb{Z}/p^{k_{r+1}}\mathbb{Z},$$
for some integers $1\leq k_1 \leq k_2\leq\dots\leq k_r\leq k_{r+1}$.
Let
$$H = \mathbb{Z}/p^{k_1}\mathbb{Z}\oplus \mathbb{Z}/p^{k_2}\mathbb{Z}\oplus \dots\oplus \mathbb{Z}/p^{k_{r}}\mathbb{Z},$$
and consider the projection map
$\pi : G\rightarrow H$.
Since $S$ generates $G$, the set $\pi(S)$ generates $H$.
Therefore by our inductive assumption, 
we can choose $x_1,\dots, x_r\in S$ such that 
$\{\pi(x_1),\dots,\pi(x_r)\}$ generates $H$.

Let $K$ be the subgroup of $G$ generated by $\{x_1,\dots,x_r\}$.
The quotient group $G/K$ is a cyclic $p$-group, so there exists an $x_{r+1}\in S$ such that $x_{r+1} + K$ generates $G/K$.
It follows that $\{x_1,\dots,x_{r+1}\}$ generate $G$.
Thus by induction the theorem is proved.
\end{proof}

Now consider the prime decomposition of $N=\exp(G)$, ie.
$$N = p_1^{m_1}p_2^{m_2}\dots p_s^{m_s},$$
where here $p_1,\dots, p_s$ are distinct primes and $m_1,\dots,m_s$ are positive integers.
For each $1\leq k\leq s$, we define $b_k = N/p_k^{m_k}$.
For each $1\leq k\leq s$, define a subgroup
$$H_{k} = \{x\in G: b_kx = 0\}.$$
As a consequence of Lagrange's Theorem, the group quotient $G/H_k$ is a group with exponent $\exp(G/H_k) = p_k^{m_k}$.
By Lemma \ref{lem:prime generating}, we can choose a $x_{k1},\dots, x_{kr_k}\in \supp(\wh f)$ with $r_k\leq r$ such that 
$$\{x_{k1}+H_k,x_{k2}+H_k,\dots, x_{kr_k}+H_k\}$$
generate $G/H_k$.

Using Theorem \ref{thm:recon-oddpow} and Theorem \ref{thm:recon-evenpow}, we can determine the values of $\wh f(x)$ for each $x\in \supp(\wh f)$, up to a root of unity in $\mathbb Q(\xi_N)$.
For each $1\leq k\leq s$ and each $1\leq j\leq m_k$, 
choose values $\beta_{kj}\in \mathbb Q(\xi_N)$ in the cyclotomic extension field with $$\beta_{kj}^N=\wh f(x_{kj})^N,\quad 1\leq k\leq s,\ 1\leq j\leq r_k.$$

The values of the $\beta_{kj}$'s that we just chose may be \textit{misaligned} in the sense that they might not respect relations between the different generators we have selected.
To mitigate this, we need to \emph{realign} them by rescaling each by an appropriate root of unity.

Consider the group homomorphism
$$\varphi_k: \mathbb Z^{r_k}\rightarrow G/H_k,\quad \varphi_k(c_1,\dots,c_{r_k}) = \sum_{j=1}^{r_k}c_jx_{kj}+H_k.$$
The kernel $K_k$ of $\varphi_k$, encodes the linear relations between the generators of $G/H_k$.

The kernel $K_k$ gives us more expressions involving $\wh f$ that we can compute explicitly.
\begin{lem}
Let $x_{kj}$, $\beta_{kj}$, and $K_k$ be as above.
Then for each $(c_1,\dots, c_{r_k})\in K_k$, we can compute the product $\prod_{j=1}^{r_k}\frac{\wh f(x_{kj})^{b_kc_j}}{\beta_{kj}^{b_kc_j}}$ using only the information of autocorrelations up to order $2r$ for $N$ odd and order $3r$ for $N$ even.
\end{lem}
\begin{proof}
For each $(c_1,\dots, c_{r_k})\in K_k$, we have
$$c_1x_{k1}+c_2x_{k2}+\dots+c_{r_k}x_{kr_k}=0\mod H_k,$$
and therefore
$$b_kc_1x_{k1}+b_kc_2x_{k2}+\dots+b_kc_{r_k}x_{kr_k}=0.$$
Now consider the unital decompositions of the coefficients in the above linear combination
$$b_kc_j = u_j+v_j+w_j.$$
Since $\wh f (x_{kj})/\beta_{kj}$ is a root of unity, we know $(\wh f (x_{kj})/\beta_{kj})^u=\sigma_u(\wh f (x_{kj})/\beta_{kj})$ for all $u\in (\mathbb Z/N\mathbb Z)^\times$.  Using this, we calculate
\begin{align*}
\prod_{j=1}^{r_k}\frac{\wh f(x_{kj})^{b_kc_j}}{\beta_{kj}^{b_kc_j}}
 & = \prod_{j=1}^{r_k}\frac{\wh f(x_{kj})^{u_j}\wh f(x_{kj})^{v_j}\wh f(x_{kj})^{w_j}}{\beta_{kj}^{u_j}\beta_{kj}^{v_j}\beta_{kj}^{w_j}}\\
 & = \prod_{j=1}^{r_k}\sigma_{u_j}\left(\frac{\wh f(x_{kj})}{\beta_{kj}}\right)\sigma_{v_j}\left(\frac{\wh f(x_{kj})}{\beta_{kj}}\right)\sigma_{w_j}\left(\frac{\wh f(x_{kj})}{\beta_{kj}}\right)\\
 & = \prod_{j=1}^{r_k}\frac{\wh f(u_jx_{kj})\wh f(v_jx_{kj})\wh f(w_jx_{kj})}{\sigma_{u_j}(\beta_{kj})\sigma_{v_j}(\beta_{kj})\sigma_{w_j}(\beta_{kj})}\\
 & = \frac{\wh M(f;(u_jx_{kj},v_jx_{kj},w_jx_{kj})_{j=1}^{r_k})}{\prod_{j=1}^{r_k}\sigma_{u_{j}}(\beta_{kj})\sigma_{v_{j}}(\beta_{kj})\sigma_{w_{j}}(\beta_{kj})}.
\end{align*}
The autocorrelations in the expression above are at most $2r_k\leq 2r$ if $N$ is odd and $3r_k\leq 3r$ if $N$ is even.
\end{proof}

Using the previous lemma, we can determine a condition for what roots of unity we must multiply the $\beta_{kj}$'s by in order to realign them.
\medskip\mbox{}\\
\noindent\textbf{Remodulation condition:}
We say the $N$'th roots of unity $\lambda_{kj}$ for $1\leq k\leq s$ and $1\leq j\leq r_k$ satisfy the \textit{remodulation condition} if for any $(c_1,\dots, c_{r_k})\in K_k$, we have
\begin{equation}\label{eqn:remodulation}
\prod_{j=1}^{r_k}\lambda_{kj}^{b_kc_{j}} = \prod_{j=1}^{r_k}\frac{\wh f(x_{kj})^{b_kc_j}}{\beta_{kj}^{b_kc_j}}.
\end{equation}

Note that by the previous Lemma, this condition can be computed explicitly using our limited knowledge of the autocorrelations.  Furthermore, there exist complex roots of unity $\lambda_{kj}$ satisfying the remodulation condition, since the roots of unity $\wh f(x_{kj})/\beta_{kj}$ themselves satisfy it.
In practice, finding the $\lambda_{kj}$'s comes down to determining a basis of the free Abelian group $K_k$, and then imposing the constraint \eqref{eqn:remodulation} on each element of the basis.

Since the greatest common divisior of the collection of integers $b_1,\dots,b_s$ is $1$, the Euclidean algorithm allows us to find a combination of integers $a_1,\dots, a_s$ with
$$a_1b_1+a_2b_2+\dots+a_sb_s=1.$$
Choose $\lambda_{kj}$'s satisfying the remodulation condition and define
\begin{equation}\label{eqn:aligned}
\alpha_{kj} = \lambda_{kj}\beta_{kj},\quad 1\leq k\leq s,\ 1\leq j\leq r_k.
\end{equation}
The $\alpha_{kj}$'s are \emph{aligned}, and the
 next theorem shows that they, combined with knowledge of the autocorrelations of $\wh f$ up to order $2r+2$ for $N$ odd or $3r+3$ for $N$ even, determine $\wh f(x)$ up to a global phase.

 We start by using Theorem \ref{thm:recon-oddpow} or Theorem \ref{thm:recon-evenpow} to construct a function $\wh h: G\rightarrow \mathbb Q(\xi_N)$ with the property that $\wh h(x)^N=\wh f(x)^N$ for all $x\in G$.  Without loss of generality, we can take $\wh h(x_{kj}) = \beta_{kj}$ for all $1\leq k\leq s$ and $1\leq j\leq r_k$.  The next theorem explains how to realign $\wh h$ to recover $\wh f$ up to multiplication by a character.

\begin{thm}\label{thm:finalconstruction}
Suppose that $f$ is a rational-valued function on $G$ and choose $a_k$'s, $b_k$'s, $\alpha_{kj}$'s, $x_{kj}$'s and $\wh h$ as specified by the discussion above.  
Define a complex-valued function $\wh g$ on $G$ whose support is the same as $\wh f$, and which is given for $x\in\supp(\wh f)$ by 
$$\wh g(x) = (\wh g(x)^{b_1})^{a_1}(\wh g(x)^{b_2})^{a_2}\dots(\wh g(x)^{b_s})^{a_s},$$
where here
$$\ol{\wh g^{b_k}(x)} = \wh C_k\left(\prod_{j=1}^{r_k} \sigma_{u_{kj}}(\alpha_{kj}) \sigma_{v_{kj}}(\alpha_{kj}) \sigma_{w_{kj}}(\alpha_{kj})\right)^{-1},$$
$$\wh C_k =\frac{\wh M(f; (-u_kx,-v_kx,-w_kx,u_{kj}x_{kj},v_{kj}x_{kj},w_{kj}x_{kj})_{j=1}^{r_k})\wh h(x)^{b_k}}{\sigma_{u_k}(\wh h(x))\sigma_{v_k}(\wh h(x))\sigma_{w_k}(\wh h(x))},$$
$$b_kc_{kj}=u_{kj}+v_{kj}+w_{kj}$$
is a unital decomposition of $b_k$ times the coefficients of a linear expansion
$$x = c_{k1}x_{k1}+c_{k2}x_{k2}+\dots+c_{kr_k}x_{kr_k}\mod H_k$$
for each $1\leq k\leq s$ (see Definition \ref{def:unital}), and
$$b_k = u_k + v_k + w_k$$
is a unital decomposition of $b_k$.

Then the function $\wh g(x)$ is well-defined, and there exists $y\in G$ with 
$$\wh g(x) = \wh f(x)\chi(x,y),\quad\text{for all}\ x\in G.$$
\end{thm}
\begin{proof}
For each $1\leq k\leq s$ and $1\leq j\leq r_k$, there exists an $N$'th root of unity $\mu_{kj}$ with 
$$\wh f(x_{kj}) = \mu_{kj}\alpha_{kj}.$$
Let $x\in\supp(\wh f)$.
Choose integers $c_{k1},\dots, c_{kr_k}$ satisfying
$$x = c_{k1}x_{k1}+c_{k2}x_{k2}+\dots+c_{kr_k}x_{kr_k}\mod H_k.$$
Then
$$b_kx = b_kc_{k1}x_{k1} + b_kc_{k2}x_{k2}+\dots+b_kc_{kr_k}x_{kr_k}.$$
Now if $u\in(\mathbb Z/N\mathbb Z)^\times$, then $\sigma_u(\wh f(x)) = \wh f(ux)$, and since $\wh f(x)/\wh h(x)$ is a root of unity,
$$\left(\frac{\wh f(x)}{\wh h(x)}\right)^u=\sigma_u\left(\frac{\wh f(x)}{\wh h(x)}\right) = \frac{\wh f(ux)}{\sigma_u(\wh h(x))},\quad \forall u\in(\mathbb Z/N\mathbb Z)^\times.$$
Therefore for any $x\in\supp(\wh f)$
\begin{align*}
\frac{\ol{\wh f^{b_k}(x)}}{\wh h(x)^{b_k}}
  & = \left(\frac{\ol{\wh f(x)}}{\wh h(x)}\right)^{u_k}\left(\frac{\ol{\wh f(x)}}{\wh h(x)}\right)^{v_k}\left(\frac{\ol{\wh f(x)}}{\wh h(x)}\right)^{w_k}\\
  & = \sigma_{u_k}\left(\frac{\ol{\wh f(x)}}{\wh h(x)}\right)\sigma_{v_k}\left(\frac{\ol{\wh f(x)}}{\wh h(x)}\right)\sigma_{w_k}\left(\frac{\ol{\wh f(x)}}{\wh h(x)}\right)\\
  & = \frac{\wh f(-u_kx)\wh f(-v_kx)\wh f(-w_kx)}{\sigma_{u_k}(\wh h(x))\sigma_{v_k}(\wh h(x))\sigma_{w_k}(\wh h(x))}\\  
  & = \frac{\wh C_k}{\wh h(x)^{b_k}}\left(\prod_{j=1}^{r_k}\wh f(u_{kj}x_{kj})\wh f(v_{kj}x_{kj})\wh f(w_{kj}x_{kj})\right)^{-1}\\
  & = \frac{\wh C_k}{\wh h(x)^{b_k}}\left(\prod_{j=1}^{r_k}\mu_{kj}^{b_kc_{kj}}\sigma_{u_{kj}}(\alpha_{kj})\sigma_{v_{kj}}(\alpha_{kj})\sigma_{w_{kj}}(\alpha_{kj})\right)^{-1}\\
  & = 
  \frac{\overline{\wh g^{b_k}(x)}}{\wh h(x)^{b_k}}\left(\prod_{j=1}^{r_k}\mu_{kj}^{b_kc_{kj}}\right)^{-1}.
\end{align*}

Therefore
$$\wh g^{b_k}(x) = \wh f^{b_k}(x) \prod_{j=1}^{r_k}\mu_{kj}^{-b_kc_{kj}}.$$
This in particular shows that the value of $\wh g^{b_k}(x)$ is independent of the choice of unital decomposition.

To finish showing $\wh g^{b_k}(x)$ is well-defined, we need to show that the value is also independent of the choice of linear expansion of $x$.
Suppose that we have another linear expansion
$$x = \wt c_{k1}x_1 + \wt c_{k2}x_2 + \dots + \wt c_{kr_k}x_{kr_k}\mod H_k.$$
Then $(c_{k1}-\wt c_{k1},\dots, c_{kr_k}-\wt c_{kr_k})\in K_k$, and the fact that the $\alpha_{kj}$'s are aligned implies
\begin{align*}
\prod_{j=1}^{r_k}\mu_{kj}^{b_k(c_{kj}-\wt c_{kj})}
  & = \prod_{j=1}^{r_k}\frac{\wh f(x_{kj})^{b_k(c_{kj}-\wt c_{kj})}}{\alpha_{kj}^{b_k(c_{kj}-\wt c_{kj})}}\\
  & = \prod_{j=1}^{r_k}\frac{\wh f(x_{kj})^{b_k(c_{kj}-\wt c_{kj})}}{\lambda_{kj}^{b_k(c_{kj}-\wt c_{kj})}\beta_{kj}^{b_k(c_{kj}-\wt c_{kj})}} = 1.
\end{align*}
Therefore
$$\wh f^{b_k}(x) \prod_{j=1}^{r_k}\mu_{kj}^{-b_kc_{kj}} = \wh f^{b_k}(x) \prod_{j=1}^{r_k}\mu_{kj}^{-b_k\wt c_{kj}},$$
showing that $\wh g^{b_k}(x)$ is well-defined.
It follows that $\wh g(x)$ itself is well-defined.
Finally, notice that the function $q: G\rightarrow\mathbb C$ defined by
$$q(x) = \prod_{k=1}^{s}\prod_{j=1}^{r_k}\mu_{kj}^{-a_kb_kc_{kj}}$$
where for each $1\leq k\leq s$
$$x = c_{k1}x_{k1}+c_{k2}x_{k2}+\dots+c_{kr_k}x_{kr_k}\mod H_k,$$
is well-defined by the same argument as above and satisfies
$$q(x+x') = q(x)q(x')\quad\text{for all}\ x,x'\in G.$$
Therefore $q$ is a character and there exists $y\in G$ with $q(x) = \chi(x,y)$ for all $x\in G$.
Furthermore,
$$\wh g(x) = \chi(x,y)\wh f(x),\quad\text{for all}\ x\in G$$
and it follows that
$$g(x) = f(x+y),\quad\text{for all}\ x\in G.$$
\end{proof}

In the construction of $G$, we only used autocorrelations up to $2r+2$ for $N$ odd or $3r+3$ for $N$ even, so this also proves our Main Theorem.

\section{Sharpness of the bound}
We conclude our paper with a collection of illuminating examples which demonstrate that we could not do any better than the number of autocorrelations that we found.

\subsection{Examples demonstrating sharpness}
We start our example section by providing a couple examples of functions which require autocorrelations of very high order in order to be reconstructed.

Our first example shows that in general for a cyclic group we will require autocorrelations up to order six in order to determine a function.

\begin{ex}\label{ex:Z6example}
 Having agreement up to the sixth autocorrelation in the even setting is necessary in the previous case.
For example, suppose $G = \mathbb{Z}/6\mathbb{Z}$.
Also fix a choice of integer $d^2\in \mathbb{Z}$ (not necessarily a perfect square).

For any integers $a,b\in\mathbb{Z}$ satisfying $d^2 = a^2+3b^2$, the function whose discrete Fourier transform is
$$\wh f(a) = \left\lbrace\begin{array}{cc}
6(a \pm \sqrt{-3}b), & a=\pm 1\\
0, & \text{otherwise}
\end{array}\right.$$
has second autocorrelations defined by
$$\wh f(k)\wh f(-k) = \left\lbrace\begin{array}{cc}
6^2d^2, & k=\pm 1\\
0 & \text{otherwise}
\end{array}\right.$$
and fourth autocorrelations defined by
$$\wh f(i)\wh f(j)\wh f(k)\wh f(-i-j-k) = 6^4d^4$$
if $|i|=|j|=|k|=|i+j+k|=1$ and $0$ otherwise.
Since $6$ is even and $\wh f(x)$ is zero for $x$ even, the third and fifth autocorrelations are identically zero.
In particular, all the autocorrelations depend solely on the on value of $d^2$.

Thus for any integers $a,b$ satisfying $a^2+3b^2=d^2$ the function
\begin{center}
\begin{tabular}{c|c|c|c|c|c|c}
$x$    & $0$ & $1$ & $2$ & $3$ & $4$ & $5$\\\hline
$f(x)$ & $2a$ & $a-3b$ & $-a-3b$ & $-2a$ & $3b-a$ & $a+3b$
\end{tabular}
\end{center}
has the same autocorrelations up to order $5$.
\end{ex}
\begin{remk}
This generalizes all of the examples found in Gr\"{u}nbaum and Moore's original paper \cite{grunbaum}.
\end{remk}

The first example was simple to create, and can be expanded easily to show that autocorrelations of order at least $3r+1$ are required to reconstruct a rational-valued function. This is done in the next example.
\begin{ex}
Let $r>1$ and consider the group $G = (\mathbb{Z}/6\mathbb{Z})^r$, and set $e_1=(1,0,0\dots,0)$, $e_2=(0,1,0,\dots, 0)$, and so on through $e_r=(0,0,0,\dots,1)$.
Let $f$ be the function whose discrete Fourier transform is defined by
$$\wh f(\pm e_j) = 1,\quad 1\leq j\leq r,$$
and $\wh f(3,3,3,\dots,3) = 1$,
with $\wh f(x) = 0$ otherwise.
Likewise, let $g$ be the function whose discrete Fourier transforms is given by
$$\wh g(\pm e_j) = 1,\quad 1\leq j\leq r,$$
and $\wh g(3,3,3,\dots,3) = -1$,
with $\wh g(x) = 0$ otherwise.

If $\xi_6$ is a primitive $6$'th root of unity, then the only automorphisms of the field extension $\mathbb{Q}(\xi_6)$ are the identity $\sigma_1$ and complex conjugation $\sigma_5$, since the only elements of $(\mathbb Z/6\mathbb Z)^\times$ are $1$ and $5$.
Therefore $\wh f(ax) = \sigma_a(\wh f(x))$ and $\wh g(ax) = \sigma_a(\wh g(x))$, for all $a\in (\mathbb Z/6\mathbb Z)^\times$.
It follows by Theorem \ref{thm:galois} that $f$ and $g$ are rational-valued.

Finally, suppose we have a zero-sum sequence
$x_1+x_2+\dots+x_n=0$ for some $n\leq 3r$.
If at least one of the $x_j's$ is not in 
$$\{e_1,e_2,\dots,e_r,(3,3,3,\dots,3)\},$$
then
$$\wh f(x_1)\dots\wh f(x_n) = 0 = \wh g(x_1)\dots\wh g(x_n).$$
Assume otherwise, and let $t$ be the number of times $(3,3,3,\dots,3)$ appears in the sequence. 
If $t$ is odd, then each of the $e_j$'s must appear at least three times.  This would force $3r+t\leq n \leq 3r$, which is a contradiction.
This means that $t$ is even, and therefore  
$$\wh f(x_1)\dots\wh f(x_n) = 1 = \wh g(x_1)\dots\wh g(x_n),$$
since all the terms in each product are ones.
Thus the autocorrelations of $\wh f$ and $\wh g$ agree up to order $3r$.

However, $g$ is not a translation of $f$ since if it were then there would exist $y = (y_1,\dots, y_r)\in G$ with 
$$\wh g(x) = \wh f(x)\chi(x,y),\quad\text{for all}\ x\in G,$$
then inserting $x=e_k$ forces $e^{2\pi iy_k/6} = 1$ for all $k$, so that $y=0$ and $\wh f = \wh g$.
This is impossible, since
$$\wh g(3,3,3,\dots,3) = -1\neq 1 = \wh f(3,3,3,\dots,3).$$
\end{ex}
Now we demonstrate a more complicated example to show that the autocorrelations up to order $3r+2$ are also insufficient to determine a rational-valued function in general.
This proves that the upper bound we found on the required number of autocorrelations is sharp.

\begin{ex}\label{ex:kahuna}
Let $p$ and $q$ be distinct odd primes, let $r>0$ be an integer, and consider the group $G = (\mathbb{Z}/2pq\mathbb{Z})^r$.
Set $e_1=(1,0,0\dots,0)$, $e_2=(0,1,0,\dots, 0)$, and so on through $e_r=(0,0,0,\dots,1)$.

Let $f$ and $g$ be the functions whose discrete Fourier transforms are defined to be zero, except for the values
$$\wh f(ap(q+2)e_k) = \wh g(ap(q+2)e_k) = 1$$
for all $1\leq k\leq r$ and $a\in(\mathbb{Z}/2pq\mathbb{Z})^\times$,
and
$$\wh f(aq(p+2),aq(p+2),\dots,aq(p+2)) = 1,$$
and
$$\wh g(aq(p+2),aq(p+2),\dots,aq(p+2)) = (-1)^a,$$
for all $a\in (\mathbb{Z}/2pq\mathbb{Z})^\times$.
By virtue of their definitions, we know
$$\wh f(ax) = \sigma_a(\wh f(x))\quad\text{and}\quad \wh g(ax) = \sigma_a(\wh g(x))$$
for all $a\in(\mathbb Z/2pq\mathbb{Z})^\times$.
Therefore, $f$ and $g$ are rational functions by Theorem \ref{thm:galois}.

The functions $f$ and $g$ are not translations of each other.
To see this, assume that they were.
Then there would exist $y=(y_1,\dots, y_r)\in G$ with
$\wh g(x) = \wh f(x)\chi(x,y)$.
Evaluating this expression at $x=e_k^{p(q+2)}$, we get that $y_k = 0\mod 2q$ for all $1\leq k\leq r$.
In particular, this means we can write $y_k=2z_k$ for some integers $z_1,\dots,z_r$.
However, evaluating $\wh g(x)^p = \wh f(x)^p\chi(x,y)^p$ at $(q(p+2),q(p+2),\dots,q(p+2))$ we then get 
$$-1 = (-1)^p = e^{2\pi i (z_1+\dots+z_r)(p+2)} = 1.$$
This is a contradiction, proving that $f$ and $g$ are not translations of each other.

The autocorrelations of $\wh f$ and $\wh g$ agree up to order $3r+2$.
To see this, suppose $x_1,\dots, x_n\in G$ with $x_1+\dots+x_n=0$.
If any of the $x_j$'s lie outside the support of $\wh f$, then $\wh f(x_1)\dots\wh f(x_n)=\wh g(x_1)\dots\wh g(x_n)=0$.
Therefore without loss of generality, we may assume $x_1,\dots,x_n\in\supp(\wh f)$.

For each $1\leq k\leq r$ and $a\in (\mathbb{Z}/2pq\mathbb{Z})^\times$, define
$$n_{ka} = \#\{x_j: x_j=ap(q+2)e_k\}$$
and
$$m_{a} = \#\{x_j: x_j=(aq(p+2),aq(p+2),\dots,aq(p+2))\}$$
and set
$$
A_k = \sum_{a\in (\mathbb{Z}/2pq\mathbb{Z})^\times} n_{ka}a
\quad\text{and}\quad
B = \sum_{a\in (\mathbb{Z}/2pq\mathbb{Z})^\times} m_{a}a.
$$
In this notation,
$$\wh f(x_1)\dots\wh f(x_n) = 1\quad\text{and}\quad\wh g(x_1)\dots\wh g(x_n) = (-1)^B.$$

If $B$ is even, then the autocorrelations agree.
Suppose instead that $B$ is odd.
Then the condition $x_1+\dots+x_n=0$ forces
$$A_k p(q+2) + B q(p+2)=0\mod 2pq$$
for each $1\leq k\leq r$.
This in turn forces $A_k$ to be an odd multiple of $q$ and $B$ to be an odd multiple of $p$.
Therefore, in order to write $A_k$ and $B$ as sums of elements of $(\mathbb{Z}/2pq\mathbb{Z})^\times$, we must use three or more elements.
In other words, $\sum_{a\in (\mathbb{Z}/2pq\mathbb{Z})^\times} n_{ka}\geq 3$
and $\sum_{a\in (\mathbb{Z}/2pq\mathbb{Z})^\times} m_{a}\geq 3.$
Hence
$$n=\sum_{a\in (\mathbb{Z}/2pq\mathbb{Z})^\times} m_{a} + \sum_{k=1}^r\sum_{a\in (\mathbb{Z}/2pq\mathbb{Z})^\times} n_{ka} \geq 3 + 3r.$$
Therefore to see this discrepancy between the autocorrelations, we must be considering a autocorrelation of order at least $3r+3$.
In particular, this proves that the autocorrelations of $f$ and $g$ will all agree up to order $3r+2$.
\end{ex}

The smallest two-dimensional case of Example \ref{ex:kahuna} is for $p=3$ and $q=5$.
In this case, we get two functions $f$ and $g$ on a $30\times 30$ grid whose autocorrelations are the same up to order $8$ (see Figure \ref{fig:Z30Z_squared_f_and_g}).
The two-dimensional discrete Fourier transforms of $f$ and $g$ are defined by
$$\wh f(k,0) = \wh f(0,k) = \wh g(k,0)=\wh g(0,k) = 1$$
for $k\in \{3,9,21,27\}$;
$$\wh f (5,5) = \wh f (25,25) = 1;$$
$$\wh g (5,5) = \wh g (25,25) = -1;$$
and $\wh f (x) = \wh g(x) = 0$ elsewhere. 
The functions $f$ and $g$ are not translations of each other.  In fact, one can see that $g$ is a translation of $-f$ by $(-5,5)\in \mathbb Z/30\mathbb Z \times\mathbb Z/30\mathbb Z$.  Therefore $f$ and $g$ are translations of each other up to a ``global phase" difference of $-1$.  The difference in global phase isn't detected until the ninth autocorrelation because the odd autocorrelations of orders $1,3,5,$ and $7$ are all zero.

\begin{figure}[htp]
\centering
\includegraphics[width=\linewidth]{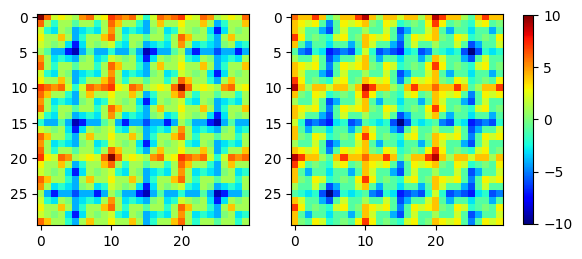}
\caption{A pixel representation of the functions $f$ (left) and $g$ (right) from Example \ref{ex:kahuna} with $p=3$ and $q=5$. These are two functions on a $30\times30$ grid whose autocorrelations agree up to order $8$, but which are not translations of each other.  In fact, $g$ is a translation of $-f$ by $(-5,5)$.  The odd-order autocorrelations vanish until order $9$, so this difference in global phase is not detected until then.}
\label{fig:Z30Z_squared_f_and_g}
\end{figure}

Functions like $f$ and $g$ with very low autocorrelations are interesting in their own right.  They may be useful in radar and communication systems \cite{song}, cryptography, studying extremal cases of low-autocorrelation discrete sequences like binary sequences \cite{packebusch} and Barker sequences \cite{ukil}.

\begin{remk}
We can adopt Example \ref{ex:kahuna} to an odd order Abelian group by considering the group $G = (\mathbb Z/\ell pq)^r$ where $\ell$ is an odd number not divisible by $p$ or $q$.
Then define $\wh f,\wh g: G\rightarrow\mathbb C$ to be zero except for
$$\wh f(ap(q+\ell)e_k)=\wh g(ap(q+\ell)e_k) = 1,$$
$$\wh f(aq(p+\ell),aq(p+\ell),\dots,aq(p+\ell)) = 1,$$
and
$$\wh g(aq(p+\ell),aq(p+\ell),\dots,aq(p+\ell)) = \omega_\ell^a,$$
for $a\in (\mathbb Z/\ell pq\mathbb Z)^\times$ where here $\omega_\ell$ is a primitive $\ell$'th root of unity.  Then the autocorrelations of the functions $f$ and $g$ will agree up to order $2r+1$, but the functions are not translations of one another.
\end{remk}

\section*{Acknowledgements}
\thanks{The research of W.R.C. has been supported by an AMS-Simons Research Enhancement Grant, and RSCA intramural grant 0359121 from CSUF.  Both authors would also like to thank Professor Kenneth Childers of CSUF for helpful discussions regarding $X$-ray crystallography and Professor F.~Alberto Gr\"{u}nbaum for initially pointing out the problem to the first author.  We would also like to thank the anonymous referees whose careful reading and suggestions greatly improved the present manuscript.}

\bibliographystyle{plain}
\bibliography{moments}

@article{harker,
  title={Phases of Fourier coefficients directly from crystal diffraction data},
  author={Harker, D. and Kasper, J.~S.},
  journal={Acta Crystallographica},
  volume={1},
  number={2},
  pages={70--75},
  year={1948},
  publisher={International Union of Crystallography}
}

@article{grunbaum,
  title={The use of higher-order invariants in the determination of generalized Patterson cyclotomic sets},
  author={Gr{\"u}nbaum, F Alberto and Moore, Calvin C.},
  journal={Acta Crystallographica Section A: Foundations of Crystallography},
  volume={51},
  number={3},
  pages={310--323},
  year={1995},
  publisher={International Union of Crystallography}
}

@article{sander,
  title={On the addition of units and nonunits mod m},
  author={Sander, J.W.},
  journal={Journal of Number theory},
  volume={129},
  number={10},
  pages={2260--2266},
  year={2009},
  publisher={Elsevier}

}

@article{patterson1,
  title={A Fourier series method for the determination of the components of interatomic distances in crystals},
  author={Patterson, A. Lindo},
  journal={Physical Review},
  volume={46},
  number={5},
  pages={372},
  year={1934},
  publisher={APS}
}

@article{patterson2,
  title={A direct method for the determination of the components of interatomic distances in crystals},
  author={Patterson, A. Lindo},
  journal={Zeitschrift f{\"u}r Kristallographie-Crystalline Materials},
  volume={90},
  number={1-6},
  pages={517--542},
  year={1935},
  publisher={De Gruyter Oldenbourg}
}

@article{patterson3,
  title={Ambiguities in the X-ray analysis of crystal structures},
  author={Patterson, A. Lindo},
  journal={Physical Review},
  volume={65},
  number={5-6},
  pages={195},
  year={1944},
  publisher={APS}
}

@book{schuijer,
  title={Analyzing atonal music: Pitch-class set theory and its contexts},
  author={Schuijer, Michiel},
  volume={60},
  year={2008},
  publisher={University Rochester Press}
}

@book{baumert,
  title={Cyclic difference sets},
  author={Baumert, Leonard D},
  volume={182},
  year={2006},
  publisher={Springer}
}

@article{pauling,
  title={The crystal structure of bixbyite and the c-modification of the sesquioxides},
  author={Pauling, Linus and Shappell, Maple D.},
  journal={Zeitschrift f{\"u}r Kristallographie-Crystalline Materials},
  volume={75},
  number={1},
  pages={128--142},
  year={1930},
  publisher={De Gruyter Oldenbourg}
}

@book{hewitt,
  title={Abstract harmonic analysis I},
  author={Hewitt, Edwin and Ross, Kenneth A.},
  volume={115},
  year={2012},
  publisher={Springer Science \& Business Media}
}

@book{fuchs,
  title={Abelian groups},
  author={Fuchs, L{\'a}szl{\'o} and Kahane, JP and Robertson, AP and Ulam, S},
  volume={960},
  year={1960},
  publisher={Springer}
}

@inproceedings{otsu,
  title={A new scheme for practical flexible and intelligent vision systems.},
  author={Otsu, Nobuyuki and Kurita, Takio},
  booktitle={MVA},
  pages={431--435},
  year={1988}
}

@INPROCEEDINGS{hidaka,
  author={{H}idaka, {A}kinori and {K}urita, {T}akio and {O}tsu, {N}obuyuki},
  booktitle={2008 Bio-inspired, Learning and Intelligent Systems for Security}, 
  title={Object Detection by Selective Integration of {HLAC} Mask Features}, 
  year={2008},
  volume={},
  number={},
  pages={46--50},
  doi={10.1109/BLISS.2008.24}}

@article{zhang,
  title={Nth-order linear algorithm for diffuse correlation tomography},
  author={Zhang, Xiaojuan and Gui, Zhiguo and Qiao, Zhiwei and Liu, Yi and Shang, Yu},
  journal={Biomedical optics express},
  volume={9},
  number={5},
  pages={2365--2382},
  year={2018},
  publisher={Optical Society of America}
}

@article{shang,
  title={A Nth-order linear algorithm for extracting diffuse correlation spectroscopy blood flow indices in heterogeneous tissues},
  author={Shang, Yu and Yu, Guoqiang},
  journal={Applied physics letters},
  volume={105},
  number={13},
  year={2014},
  publisher={AIP Publishing}
}

@article{mclaughlin,
  title={Nth-order autocorrelations in pattern recognition},
  author={McLaughlin, John A. and Raviv, Josef},
  journal={Information and Control},
  volume={12},
  number={2},
  pages={121--142},
  year={1968},
  publisher={Elsevier}
}

@article{chazan,
  title={Higher order autocorrelation functions as translation invariants},
  author={Chazan, D. and Weiss, B.},
  journal={Information and Control},
  volume={16},
  number={4},
  pages={378--383},
  year={1970},
  publisher={Academic Press}
}

@book{hungerford,
  title={Algebra},
  author={Hungerford, Thomas W.},
  volume={73},
  year={2012},
  publisher={Springer Science \& Business Media}
}

@article{ferguson,
  title={Analysis of {PSLQ}, an integer relation finding algorithm},
  author={Ferguson, Helaman and Bailey, David and Arno, Steve},
  journal={Mathematics of Computation},
  volume={68},
  number={225},
  pages={351--369},
  year={1999}
}

@manual{mpmath,
  key     = {mpmath},
  author  = {The mpmath development team},
  title   = {mpmath: a {P}ython library for arbitrary-precision floating-point arithmetic (version 1.3.0)},
  note    = {{\tt http://mpmath.org/}},
  year    = {2023},
}

@article{yellott,
  title={Uniqueness properties of higher-order autocorrelation functions},
  author={Yellott Jr, John I. and Iverson, Geoffrey J.},
  journal={Journal of the Optical Society of America A},
  volume={9},
  number={3},
  pages={388--404},
  year={1992},
  publisher={Optical Society of America}
}

@article{toyoda,
  title={Extension of higher order local autocorrelation features},
  author={Toyoda, Takahiro and Hasegawa, Osamu},
  journal={Pattern Recognition},
  volume={40},
  number={5},
  pages={1466--1473},
  year={2007},
  publisher={Elsevier}
}

@article{yamamoto,
  title={A design of higher order auto-correlation vision chip},
  author={YAMAMOTO, Kenkichi and ISHII, Idaku},
  journal={IEICE transactions on information and systems},
  volume={86},
  number={8},
  pages={1475},
  year={2003}
}

@article{song,
  title={Optimization methods for designing sequences with low autocorrelation sidelobes},
  author={Song, Junxiao and Babu, Prabhu and Palomar, Daniel P},
  journal={IEEE Transactions on Signal Processing},
  volume={63},
  number={15},
  pages={3998--4009},
  year={2015},
  publisher={IEEE}
}

@article{ukil,
  title={Low autocorrelation binary sequences: Number theory-based analysis for minimum energy level, Barker codes},
  author={Ukil, Abhisek},
  journal={Digital Signal Processing},
  volume={20},
  number={2},
  pages={483--495},
  year={2010},
  publisher={Elsevier}
}

@article{packebusch,
  title={Low autocorrelation binary sequences},
  author={Packebusch, Tom and Mertens, Stephan},
  journal={Journal of Physics A: Mathematical and Theoretical},
  volume={49},
  number={16},
  pages={165001},
  year={2016},
  publisher={IOP Publishing}
}

@article{hamburger,
  title={{\"U}ber eine erweiterung des stieltjesschen momentenproblems},
  author={Hamburger, Hans},
  journal={Mathematische Annalen},
  volume={81},
  number={2},
  pages={235--319},
  year={1920},
  publisher={Springer}
}

@article{stieltjes1,
  author    = {Stieltjes, Thomas Joannes},
  title     = {Recherches sur les fractions continues},
  journal   = {Annales de la Facult{\'e} des sciences de Toulouse:
               Math{\'e}matiques},
  series    = {1},
  volume    = {8},
  number    = {4},
  pages     = {J1--J122},
  year      = {1894},
  doi       = {10.5802/afst.108},
  language  = {French}
}

@article{stieltjes2,
  author    = {Stieltjes, Thomas Joannes},
  title     = {Recherches sur les fractions continues, suite et fin},
  journal   = {Annales de la Facult{\'e} des sciences de Toulouse:
               Math{\'e}matiques},
  series    = {1},
  volume    = {9},
  number    = {1},
  pages     = {A5--A47},
  year      = {1895},
  doi       = {10.5802/afst.109},
  language  = {French}
}

@book{shohat,
  author    = {Shohat, James Alexander and Tamarkin, Jacob David},
  title     = {The Problem of Moments},
  series    = {Mathematical Surveys},
  volume    = {1},
  publisher = {American Mathematical Society},
  address   = {New York},
  year      = {1943},
  pages     = {144},
  doi       = {10.1090/surv/001}
}

@article{mead,
  title={Maximum entropy in the problem of moments},
  author={Mead, Lawrence R and Papanicolaou, Nikos},
  journal={Journal of Mathematical Physics},
  volume={25},
  number={8},
  pages={2404--2417},
  year={1984},
  publisher={American Institute of Physics}
}

@book{schmudgen,
  title={The moment problem},
  author={Schm{\"u}dgen, Konrad and others},
  volume={9},
  year={2017},
  publisher={Springer}
}

@article{provost,
  title={Moment-based density approximants},
  author={Provost, Serge B},
  journal={Mathematica Journal},
  volume={9},
  number={4},
  pages={727--756},
  year={2005},
  publisher={Redwood City, Ca.: Advanced Book Program, Addison-Wesley Pub. Co., c1990-}
}

@phdthesis{Kakarala,
  author = {Kakarala, Ramakrishna},
  title  = {Triple Correlation on Groups},
  school = {University of California, Irvine},
  year   = {1992},
  type   = {Ph.D. thesis},
  note   = {UMI Order No. GAX93-04094}
}

@article{adler,
  title={A note on translation invariants},
  author={Adler, Roy L and Konheim, Alan G},
  journal={Proceedings of the American Mathematical Society},
  volume={13},
  number={3},
  pages={425--428},
  year={1962},
  publisher={JSTOR}
}

@article{kakarala2,
  title={The bispectrum as a source of phase-sensitive invariants for Fourier descriptors: a group-theoretic approach},
  author={Kakarala, Ramakrishna},
  journal={Journal of Mathematical Imaging and Vision},
  volume={44},
  number={3},
  pages={341--353},
  year={2012},
  publisher={Springer}
}

@article{bartelt,
  author  = {Bartelt, H. O. and Lohmann, A. W. and Wirnitzer, B.},
  title   = {Phase and Amplitude Recovery from Bispectra},
  journal = {Applied Optics},
  year    = {1984},
  volume  = {23},
  number  = {18},
  pages   = {3121--3129},
  doi     = {10.1364/AO.23.003121}
}

@article{domokos1,
  title={Degree bound for separating invariants of abelian groups},
  author={Domokos, M{\'a}ty{\'a}s},
  journal={Proceedings of the American Mathematical Society},
  volume={145},
  number={9},
  pages={3695--3708},
  year={2017}
}

@article{domokos2,
  title={Separating polynomial invariants over non-closed fields of finite abelian groups},
  author={Domokos, M{\'a}ty{\'a}s},
  journal={Applied and Computational Harmonic Analysis},
  pages={101877},
  year={2026},
  publisher={Elsevier}
}

@article{bandeira,
  title={Estimation under group actions: recovering orbits from invariants},
  author={Bandeira, Afonso S and Blum-Smith, Ben and Kileel, Joe and Niles-Weed, Jonathan and Perry, Amelia and Wein, Alexander S},
  journal={Applied and Computational Harmonic Analysis},
  volume={66},
  pages={236--319},
  year={2023},
  publisher={Elsevier}
}

@article{kemper,
  title={Separating invariants},
  author={Kemper, Gregor},
  journal={Journal of Symbolic Computation},
  volume={44},
  number={9},
  pages={1212--1222},
  year={2009},
  publisher={Elsevier}
}

@article{kohls,
  title={Degree bounds for separating invariants},
  author={Kohls, Martin and Kraft, Hanspeter},
  journal={Mathematical Research Letters},
  volume={17},
  number={6},
  pages={1171--1182},
  year={2010},
  publisher={International Press of Boston, Inc. Somerville, MA 02143, USA}
}

@inproceedings{edidin,
  author    = {Dan Edidin and Josh Katz},
  title     = {Orbit Recovery from Invariants of Low Degree in Representations of Finite Groups},
  booktitle = {2025 International Conference on Sampling Theory and Applications (SampTA)},
  pages     = {1--5},
  year      = {2025},
  doi       = {10.1109/SampTA64769.2025.11133516}
}

@article{perry,
  author  = {Amelia Perry and Jonathan Weed and Afonso S. Bandeira and Philippe Rigollet and Amit Singer},
  title   = {The Sample Complexity of Multireference Alignment},
  journal = {SIAM Journal on Mathematics of Data Science},
  volume  = {1},
  number  = {3},
  pages   = {497--517},
  year    = {2019},
  doi     = {10.1137/18M1214317}
}

@article{dym,
  author  = {Nadav Dym and Steven J. Gortler},
  title   = {Low-Dimensional Invariant Embeddings for Universal
             Geometric Learning},
  journal = {Foundations of Computational Mathematics},
  volume  = {25},
  pages   = {375--415},
  year    = {2025},
  doi     = {10.1007/s10208-024-09641-2}
}

\end{document}